\documentclass[preprint,3p]{elsarticle}

\usepackage{setspace}
\usepackage{amsmath, amsfonts,amssymb}
\usepackage{multicol}
\usepackage{color}
\usepackage{xspace}
\usepackage{enumerate}

\usepackage{appendix}
\usepackage{pifont}
\usepackage{subcaption}
\usepackage{tikz}
\usepackage{float}
\usepackage{tcolorbox}
\usetikzlibrary{decorations.markings}
\graphicspath{ {images/} }
\usepackage{arydshln}
\usepackage[algo2e,boxed]{algorithm2e}
\usepackage{hyperref}
\hypersetup{pdfauthor=author}
\usepackage{cleveref}
\usepackage{natbib}
\newcommand{\R}{\mathbb{R}}

\newcommand{\comm}[1]{}
\renewcommand{\Phi}{\varPhi}
\renewcommand{\phi}{\varphi}
\newcommand{\xmark}{\ding{55}}
\vbadness=\maxdimen

\newtheorem{example}{Example}
\newtheorem{theorem}{Theorem}
\newtheorem{lemma}[theorem]{Lemma} 
 
\newtheorem{remark}[theorem]{Remark}
\newtheorem{corollary}[theorem]{Corollary}
\newtheorem{definition}[theorem]{Definition}

\newcommand{\BlackBox}{\rule{1.5ex}{1.5ex}}  
\newenvironment{proof}{\par\noindent{\bf Proof\ }}{\hfill\BlackBox}

\long\def\acks#1{\vskip 0.3in\noindent{\large\bf Acknowledgments}\vskip 0.2in
\noindent #1}

\newlength{\x}
\newlength{\y}

\RequirePackage[normalem]{ulem} 
\RequirePackage{color}\definecolor{RED}{rgb}{1,0,0}\definecolor{BLUE}{rgb}{0,0,1} 
\providecommand{\DIFadd}[1]{{\protect\color{blue}\uwave{#1}}} 
\providecommand{\DIFaddbegin}{} 
\providecommand{\DIFaddend}{} 
\RequirePackage{listings} 
\RequirePackage{color} 
\lstdefinelanguage{DIFcode}{ 
  moredelim=[il][\color{red}\sout]{\%DIF\ <\ }, 
  moredelim=[il][\color{blue}\uwave]{\%DIF\ >\ } 
} 
\lstdefinestyle{DIFverbatimstyle}{ 
	language=DIFcode, 
	basicstyle=\ttfamily, 
	columns=fullflexible, 
	keepspaces=true 
} 
\lstnewenvironment{DIFverbatim}{\lstset{style=DIFverbatimstyle}}{} 
\lstnewenvironment{DIFverbatim*}{\lstset{style=DIFverbatimstyle,showspaces=true}}{} 

\begin{document}

\begin{frontmatter}

\title{Possibility Results for Graph Clustering: A Novel Consistency Axiom}

\author[1]{Fabio Strazzeri\corref{cor}}
\ead{fabio.strazzeri@upc.edu}

\author[2,3,4]{Rub\'en J.~S\'anchez-Garc\'ia}
\ead{R.Sanchez-Garcia@soton.ac.uk}

\cortext[cor]{Corresponding author.}

\address[1]{Institut de Rob\`otica i Inform\`atica Industrial, CSIC-UPC, C/ Llorens i Artigas 4-6, Barcelona 08028, Spain}
\address[2]{Mathematical Sciences, University of Southampton,  Southampton SO17 1BJ, UK}
\address[3]{Institute for Life Sciences, University of Southampton, Southampton, SO17 1BJ, UK}
\address[4]{The Alan Turing Institute, London, NW1 2DB, UK}

\begin{abstract}
\footnotesize
Kleinberg introduced three natural clustering properties, or axioms, and showed they cannot be simultaneously satisfied by any clustering algorithm. We present a new clustering property, Monotonic Consistency, which avoids the well-known problematic behaviour of Kleinberg's Consistency axiom, and the impossibility result. 
Namely, we describe a clustering algorithm, Morse Clustering, inspired by Morse Theory in Differential Topology, which satisfies Kleinberg's original axioms with Consistency replaced by Monotonic Consistency. Morse clustering uncovers the underlying flow structure on a set or graph and returns a partition into trees representing basins of attraction of critical vertices.  
We also generalise Kleinberg's axiomatic approach to sparse graphs, showing an impossibility result for Consistency, and a possibility result for Monotonic Consistency and Morse clustering. 
\end{abstract}


\end{frontmatter}
\section*{Introduction}\label{sec:Introduction}

Given a set of objects and a pairwise similarity function, a \emph{clustering algorithm} is a formal procedure that groups together objects which are similar and separate the ones which are not \cite{Jain:1988}, mimicking the human ability to categorize and group together objects by similarity. Methods and approaches to clustering algorithms have been growing for decades \DIFaddbegin \textit{
\DIFadd{ \mbox{
\cite{Jain:1988, Jain:1999, aggarwal:2013, ZHANG2022108428, TAREKEGN2021107965}}\hspace{0pt}
}}\DIFaddend, with clustering becoming a standard data analytic technique \cite{jain:2010}. 
This has been complemented by an interest in underlying principles and general desirable properties (sometimes called \emph{axioms}) of clustering algorithms  \cite{fisher:1971}, especially as clustering is an infamously ill-defined problem in the abstract \cite{jain:2010,luxburg:2012}. 

A more recent interest in the axiomatic approach was sparked by Kleinberg's impossibility theorem \cite{kleinberg:2003}. In the spirit of Arrow's impossibility theorem in social science \cite{arrow:1950}, Kleinberg gives three natural properties a clustering algorithm should have, namely Scale Invariance, Richness, and Consistency, then proves that they cannot be simultaneously satisfied. Scale Invariance guarantees that the output of the clustering algorithm remains the same if we multiply (scale) all distances by a factor $\alpha>0$, and similarly for Consistency, when we decrease intra-cluster distances and increase inter-cluster distances. Finally, Richness is the property that guarantees that an arbitrary partition of any set $X$ can be achieved by the algorithm, for a suitably defined distance function on $X$ (see Section \ref{sec:Critique} for formal statements). 

Several authors have since criticised Kleinberg's approach, particularly the Consistency axiom \cite{bendavid:2009,ackerman:2010,correa:2013}, and proposed alternative frameworks that circumvent the impossibility result. For instance, by restricting clustering functions to $k$-partitions, for a fixed $k$, the axioms can coexist \cite{zadeh:2012}; if we allow arbitrary parameters, Kleinberg's axioms are compatible when applied to a parametric family of a clustering algorithm, as discussed in \cite{correa:2013}; and, by replacing partitions by dendrograms as the output of a clustering function, the authors in \cite{carlsson:2010} show a possibility and uniqueness result satisfied by single-linkage hierarchical clustering.
In all these cases, Kleinberg's impossibility is avoided by either restricting or extending the definition of clustering function. Other authors shift the axiomatic focus to clustering quality measures
\DIFaddbegin \textit{
\DIFadd{ \mbox{
\cite{bendavid:2009,laarhoven:2014,yu:2014,NORONHA2022108612}}\hspace{0pt}
}}\DIFaddend, or cost functions \cite{karayiannis:1999,puzicha:2000}. 


In this article, we remain close to Kleinberg's original setting and directly address the problematic behaviour of the Consistency axiom instead. Namely, we replace Kleinberg's Consistency by a weaker condition that we call \emph{Monotonic Consistency}, where the rate of expansion, respectively contraction, of inter-, respectively intra-, cluster distances is not arbitrary, but globally controlled by an expansive function $\eta$ (Section \ref{sec:Expansive}). In essence, $\eta$ controls the inter-cluster expansion, while its inverse $\eta^{-1}$ controls the intra-cluster contraction. As $\eta$ is a function on distances, not pairs of points, the control is global, with points at similar distances experiencing the same expansion or contraction. Without such global condition, we would recover Outer or Inner Consistency, each incompatible with Scale Invariance and Richness \cite{ackerman:2010}.

Monotonic Consistency avoids the problematic behaviour of Consistency (see Section \ref{sec:AvoidanceProblematic}), and, unlike Consistency, it is compatible with the other two axioms (Corollary \ref{cor:Possibility}). As far as we know, this is the only alternative in the literature to the Consistency axiom that is compatible with Richness and Scale Invariance without modifying the definition of clustering function.

Our possibility result relies on a clustering method that we call \emph{Morse Clustering}, inspired by Morse Theory in Differential Topology. Although naturally a vertex-weighted clustering algorithm (in the sense of \cite{ackerman:2016}), an unweighted version (which we call \emph{agnostic} Morse Clustering) satisfies Kleinberg's original axioms, with Consistency replaced by Monotonic Consistency. 



We present three instances of Morse Clustering, corresponding to three choices of vertex and edge preorders, then show that each of them satisfy a pair of Kleinberg's original axioms, and that all of them satisfy Monotonic Consistency (Section \ref{sec:ThreeInstances}). In particular, one of them satisfy Monotonic Consistency, Scale Invariance and Richness, which are therefore mutually compatible clustering axioms (Corollary \ref{cor:Possibility}). 

Our last contribution is a generalisation of Kleinberg's impossibility result to graph clustering (Section \ref{sec:Impossibility}). A distance function $d$ on a set $X$ can be represented by a complete graph $G$ with vertex set $X$ and edges weighted by $d(u,v)>0$. In fact, many clustering algorithms (including Morse Clustering) work on this graph representation. A classical example is Single Linkage, which, in fact, only depends on a minimum spanning tree of $G$ \cite{gower:1969}. A natural generalisation of Kleinberg's setting is, therefore, the case when $G$ is an arbitrary, rather than complete, graph. That is, we fix a graph $G$ and consider distances supported on the edge set (this is the natural setting of graph clustering \cite{schaeffer:2007}). In Section \ref{sec:Impossibility}, we justify this approach, consider Kleinberg's axioms in this graph clustering setting, show that the impossibility result still holds, even when Richness is relaxed naturally to Connected-Richness (partitions where every cluster is a connected subgraph), and give a possibility result for Monotonic Consistency and the same instance of Morse Clustering. 
Our result contains the original impossibility theorem \cite{kleinberg:2003} as a particular case, and, we argue, provides the appropriate impossibility result in the context of graph clustering. 

\textit{Related work.} Kleinberg's impossibility result applies to generic clustering algorithms encoded as arbitrary functions
\begin{equation*}
    F \colon \{ d \text{ distance on } X \} \longrightarrow \{ \mathcal{P} \text{ partition on } X \},
\end{equation*}
where $X$ is a non-empty, finite set (see Section \ref{sec:Critique}). All extensions of Kleinberg's work either restrict or extend this definition of clustering function, and none addresses the problematic Consistency axiom without either modifying the definition of clustering algorithm, or the other two axioms. In \cite{zadeh:2012}, the authors restrict the codomain of $F$ to $k$-partitions, for fixed $k$,
\begin{equation*}
    F \colon \{ d \text{ distance on } X \} \longrightarrow \{ \mathcal{P} \text{ $k$-partition on } X \}.
\end{equation*}
This solves the problematic behaviour of Consistency at the cost of fixing the number of clusters \textit{a-priori}, effectively substituting Kleinberg's Richness axiom by $k$-Richness. This assumes that each clustering algorithm has a target number of clusters (which may not always be the case), and separates clustering algorithms by target cluster number. 
In \cite{correa:2013}, the author proves several possibility results for parametric clustering, that is, the domain of the clustering functions is extended to include additional input parameters
\begin{equation*}
    F \colon \{ d \text{ distance on } X \} \times \{ parameters \} \longrightarrow \{ \mathcal{P} \text{ partition on } X \}.
\end{equation*}
The approach in \cite{carlsson:2010} goes somewhat in the opposite direction: it only considers clustering algorithms that depend on the distance $d$ alone in a way that is called \emph{functorial}. An immediate drawback is that standard clustering algorithms such as $k$-means or spectral clustering are not functorial and thus excluded from their analysis. Additionally, the authors change the codomain of a clustering algorithm function from partitions to hierarchical arrangements of clusters (dendrograms),
\begin{equation*}
    F \colon \{ d \text{ distance on } X \} \longrightarrow \{ \mathcal{P} \text{ dendrogram on } X \}, \quad \text{and $F$ functorial}.
\end{equation*}
Other line of related work \cite{bendavid:2009, laarhoven:2014, puzicha:2000} focuses on clustering quality measures instead, that is, functions that assign a score to a partition of a data set. In this approach, the axioms refer to the clustering quality functions rather than to the clustering algorithms themselves.

In contrast to the above, our approach retains Kleinberg's simplicity, by keeping his original definition of clustering algorithm, and changes the Consistency axiom only, in a way directly motivated by its problematic behaviour (see Figure \ref{fig:1} and Section \ref{sec:AvoidanceProblematic}), into a compatible axiom. Indeed, several authors have criticised Kleinberg's original Consistency axiom along these lines \cite{bendavid:2009, ackerman:2010, correa:2013}. The authors in \cite{ackerman:2010}, for instance, argue that Consistency `may sound desirable and natural', however it `may be viewed as the main weakness of Kleinberg’s impossibility result'. 

\textit{Overview of results.} We define a new clustering property (or axiom) called \emph{Monotonic Consistency} (Definition \ref{def:MonotonicConsistent}), and describe it in terms of expansive and contractive maps (Section \ref{sec:Expansive}) and monotonic transformations (Sections \ref{sec:MonotonicTransformations} and \ref{sec:Characterisation}). We explicitly show how the problematic behaviour of Kleinberg's Consistency axiom is avoided by Monotonic Consistency (Section \ref{sec:AvoidanceProblematic}). We then show that Monotonic Consistency is compatible with the other two Kleinberg's original axioms. Namely, we describe a family of clustering functions (Sections \ref{sec:MorseFlowAlgorithm} and \ref{sec:MorseClusteringAlgorithm}), which we call \emph{Morse clustering}, that satisfy, in three different instances, each pair or Kleinberg's axioms, as well as the three axioms when Consistency is replaced by Monotonic Consistency (Section \ref{sec:Possibility}, Corollary \ref{cor:Possibility}). In Section \ref{sec:Impossibility}, we generalise our results to graph clustering. First, we generalise Kleinberg's original impossibility theorem to graph clustering (Section \ref{sec:ImpossibilityGraph}, Theorem \ref{thm:ImpossibilityGraphs}), then we prove a possibility theorem for Monotonic Consistency and an instance of Morse Clustering (Section \ref{sec:ImpossibilityGraph}, Theorem \ref{thm:PossibilityGraph}).

\section{Monotonic Consistency}\label{sec:MonotonicConsistency}
In this section, we introduce a weakening of the Consistency axiom that we call \emph{Monotonic Consistency}. We start with a review of Kleinberg's original axioms and the problematic behaviour of Consistency. 

\subsection{A critique of Kleinberg's axioms}\label{sec:Critique}
Given a set $X$ of $n$ objects that we want to compare, a \emph{dissimilarity} on $X$ is a pairwise function 
$$
	d: X\times X\to \R
$$
such that $d(i,j) = d(j,i)\ge 0$, and $d(i,j) = 0$ if and only if $i = j$, for all $i,j \in X$. We will adhere to the convention in the literature and refer to $d$ from now on as a \emph{distance}, although it may not satisfy the triangle inequality. Following \cite{kleinberg:2003}, we define a \emph{clustering algorithm} on $X$ as a map
\begin{equation}\label{eqn:clus-algo-Kleinberg}
	F: \{d\mbox{ distance on }X\} \to \{\mathcal P \mbox{ partition of } X\}.
\end{equation}
A \emph{partition} of $X$ is a disjoint union $X=X_1\cup\ldots\cup X_k$, and we call each $X_i$ a \emph{cluster} of the partition. If $\mathcal{P} =\{ X_1, \ldots, X_n \}$ is a partition of $X$ and $x, y \in X$, we use the notation $x \sim_\mathcal{P} y$ if $x$ and $y$ belong to the same cluster of $\mathcal{P}$, and $x \not\sim_\mathcal{P} y$ if not.  

Kleinberg \cite{kleinberg:2003} introduced three natural properties for a clustering algorithm, then proved that they cannot be simultaneously satisfied by any clustering algorithm $F$. These properties are:
\begin{itemize}
 \item \textbf{Scale Invariance}: Given a distance $d$ on $X$ and $\alpha>0$, we have $F(d) = F(\alpha\cdot d)$;
 \item \textbf{Richness}: Given a partition ${\mathcal P}$ of $X$, there exists a distance $d$ on $X$ such that $F(d) = {\mathcal P}$;
  \item \textbf{Consistency}: Given two distances $d$ and $d'$ on $X$ with $\mathcal{P} = F(d)$, if $d'$ is a \emph{${\mathcal P}$-transformation of $d$}, that is,
 \begin{eqnarray}\label{eq:1}
 \begin{cases}
 d'(v,w)\leq d(v,w) & \text{if } v \sim_\mathcal{P} w, \text{ and} \\
 d'(v,w)\geq d(v,w) & \text{if } v \not\sim_\mathcal{P} w,
 \end{cases}
 \end{eqnarray}
then $F(d)=F(d')$. 
\end{itemize}

Kleinberg also showed that each pair of these properties can be simultaneously satisfied, in fact by three different versions of Single Linkage. 

Our first contribution is a weakening of the Consistency property which is both very natural, and can coexist with Richness and Scale-Invariance. 
To motivate our definition, we first discuss the problematic behaviour of Kleinberg's Consistency in the presence of Richness and Scale Invariance (see also \cite{correa:2013,ackerman:2010,zadeh:2012}). Given $F$ a consistent and scale-invariant clustering algorithm, and two different partitions $F(d_1) \neq F(d_2)$, it can be shown \cite[Theorem 3.1]{kleinberg:2003} that each partition is not the refinement of the other (a partition $\mathcal P$ is a \emph{refinement} of $\mathcal Q$ if each cluster of $\mathcal P$ is contained in a cluster of $\mathcal Q$). In particular, given a distance $d$ and associated partition ${\mathcal P} = F(d)$, we can never obtain a partition identical to ${\mathcal P}$ but with one, or more, of its clusters further subdivided (Fig.~\ref{fig:1}). On the other hand, consider any distance $d'$ satisfying
$$
\begin{cases}
 d'(v,w) < d(v,w) & \text{if } v,w\in C_1,\\
 d'(v,w) < d(v,w) & \text{if } v,w\in C_2,\\
 d'(v,w) = d(v,w) & \mbox{otherwise,}
\end{cases}
$$
where $C$ is a cluster of $\mathcal P$ and $C=C_1\cup C_2$ is an arbitrary partition of $C$. Note that any such $d'$ is a $\mathcal{P}$-transformation of $d$. This means that we can arbitrarily emphasize the subcluster structure, to the point that it could be more natural to consider $C_1$ and $C_2$ as separate clusters (Fig.~\ref{fig:1}), while Consistency implies $F(d)=F(d')$ regardless.

\begin{figure}
 \centering
 \includegraphics[width = 0.4\textwidth]{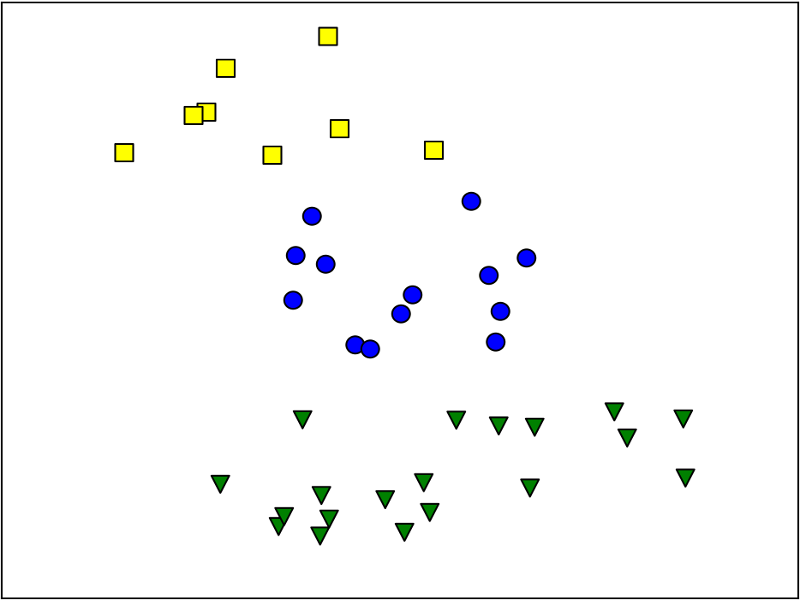} \hspace{2cm}
 \includegraphics[width = 0.4\textwidth]{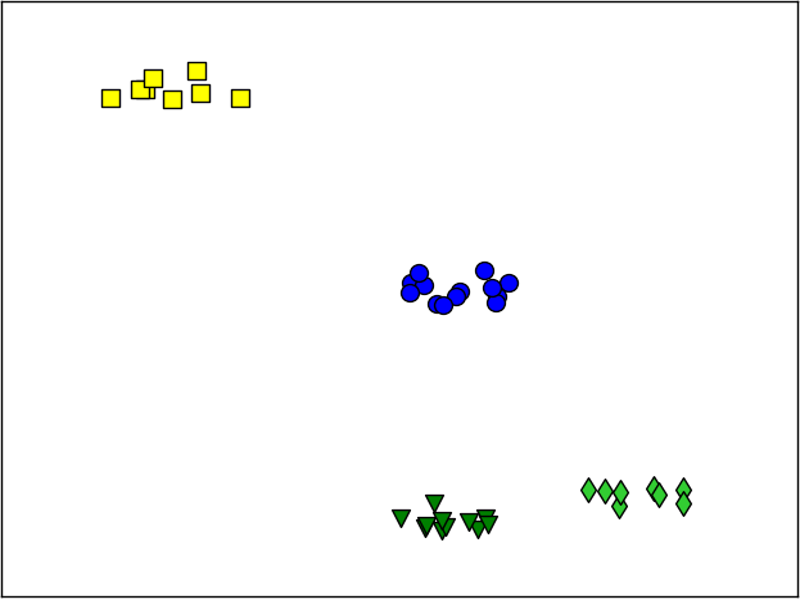}
 \caption{Problematic behaviour of the Consistency axiom. We can arbitrarily emphasize any subcluster structure without affecting the output of the clustering algorithm.  (Left) A point cloud representing the pairwise distances in a set and the output of some clustering algorithm $F$ intro three clusters. (Right) New internal structure emphasizing a subdivision of the third cluster (top to bottom). If $F$ satisfies the Consistency axiom, its output will be the same in both cases. This behaviour is explicitly avoided by Monotonic Consistency (Section \ref{sec:AvoidanceProblematic}).} \label{fig:1}
\end{figure}

We propose a more restrictive definition of Consistency which avoids this type of behaviour. The idea is to globally fix the rate at which we can increase (decrease) the intra-cluster (inter-cluster) distances. We do this restricting to $\mathcal{P}$-transformations obtained through a particular class of functions, which we describe next. 

\subsection{Expansive and contractive maps}\label{sec:Expansive}
\begin{definition}\label{def:expansive}
Let $X$ and $Y$ be subsets of $\R$. 
We call a continuous map $\eta \colon X \rightarrow Y$ \emph{expansive} if 
\begin{equation}\label{eqn:expansive}
	|\eta(x) - \eta(y)| \ge |x-y| \quad \text{for all } x,y \in X.
\end{equation}
By reversing the inequality, we define a \emph{contractive} map. 
\end{definition}
Expansive maps can be defined more generally for maps between metric spaces \cite{gottschalk:1955} as maps that do not decrease distances between pairs of points, and we have added the continuity hypothesis for convenience (see Remark \ref{rmk:noabs}). We will use expansive maps to expand and contract distances with respect to a partition, namely, $d'(u,v)=\eta(d(u,v))$ if $u$ and $v$ belong to different clusters, and $d'(u,v)=\eta^{-1}(d(u,v))$ if they belong to the same cluster. In particular, we take $X=Y=[0,\infty)$ in the definition above, and assume $\eta(0)=0$. The following lemma summarises some useful properties. 

\begin{lemma}\label{lem:expansive-properties}
Let $\eta \colon [0,\infty) \to [0,\infty)$ be a continuous expansive map with $\eta(0)=0$. Then: 
\begin{itemize}
\item[(i)] $\eta$ is strictly increasing, a bijection, and satisfies $\eta(x) \ge x$ for all $x$; 
\item[(ii)] $\eta^{-1}$ is strictly increasing, a contractive map, and satisfies $\eta^{-1}(x) \le x$ for all $x$.
\end{itemize}
\end{lemma}

\begin{proof}
(i) By contradiction, if $\eta$ is not strictly increasing, we can find $x > y$ with $\eta(x) \le \eta(y)$, so that $\eta(0) = 0 \le \eta(x) \le \eta(y)$ and, by the Intermediate Value Theorem, we can find $z \in [0,y]$ such that $\eta(z)=\eta(x)$, a contradiction. The  growth condition is immediate from the expansion property \eqref{eqn:expansive} for $y=0$,
\[
    \lvert\eta (x)\lvert = \eta(x)\ge\lvert x\lvert\ = x, 
\]
for all $x \in [0,\infty)$. Since $\eta$ is strictly increasing, it is injective. It is also surjective: The growth condition above gives $\eta(x) \to \infty$ as $x \to \infty$ and, together with $\eta(0)=0$ and continuity, we have that $\eta$ takes all values in $[0,\infty)$.\medskip

(ii) Since $\eta$ is bijective, it has an inverse $\eta^{-1}$. The inverse of a (strictly) increasing function is also (strictly) increasing. To show this, and the two remaining properties, one can simply use the corresponding properties of $\eta$ in (i) on $x'=\eta(x)$ and $y'=\eta(y)$.
\end{proof}

\begin{example}\label{ex:eta}
The following are examples of expansive functions $\eta \colon [0,\infty) \to [0,\infty)$ with $\eta(0)=0$. 
\begin{enumerate}
\item (Linear) $\eta(x) = \alpha x$ for $\alpha \ge 1$ (Fig.~\ref{fig:examples_A}). 
\item (Piecewise linear) $\eta=\eta(d_i) + \alpha_i (x-d_i)$ for $x\in[d_i,d_{i+1}]$, where $0 = d_1 < d_2 < \ldots < d_n$, $\eta(0)=0$, and $\alpha_i \ge 1$, for all $i$ (Fig.~\ref{fig:examples_B}) .
\item (Differentiable) A differentiable function $\eta \colon [0,\infty) \to [0,\infty)$ with $\eta(0)=0$ is expansive if and only if $\eta'(x) \ge 1$ for all $x$ (Fig.~\ref{fig:examples_C}).
\item (Graphical criterion) A continuous function $\eta \colon [0,\infty) \to [0,\infty)$ is expansive if and only if the function $\eta(x) - x$ is increasing (this follows from Remark \ref{rmk:noabs}).
\end{enumerate}

\begin{figure}[H]
 \centering
 \setlength{\x}{0.48\linewidth}
 \setlength{\y}{0.7\x}
 \begin{subfigure}{\x}
  \centering
  \includegraphics[width = \y]{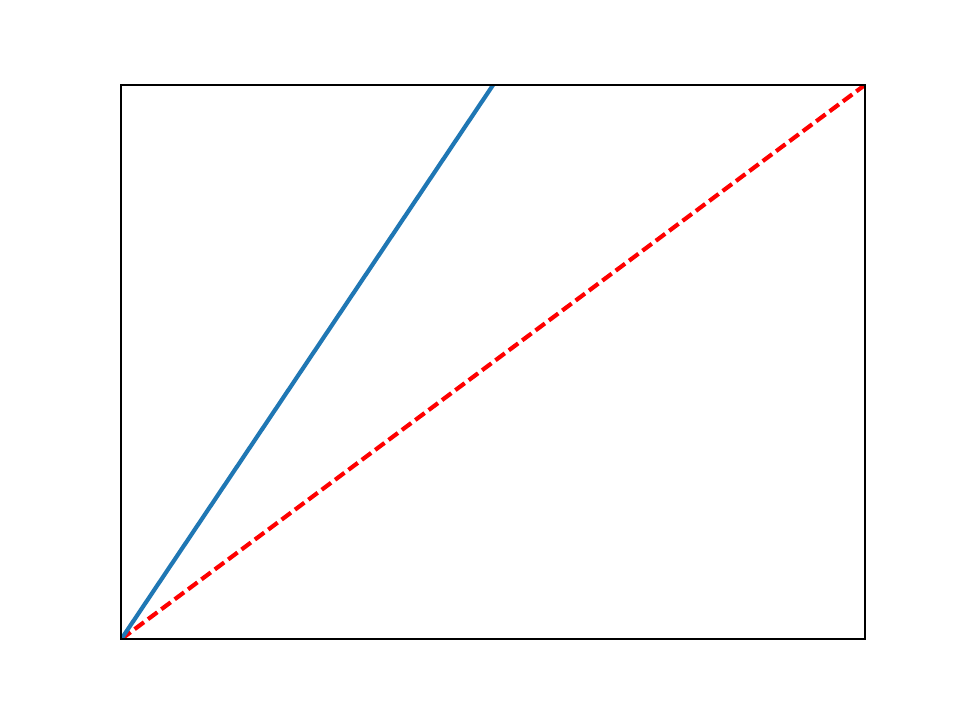}
  \caption{Linear}\label{fig:examples_A}
 \end{subfigure}
 \begin{subfigure}[H]{\x}
  \centering
  \includegraphics[width = \y]{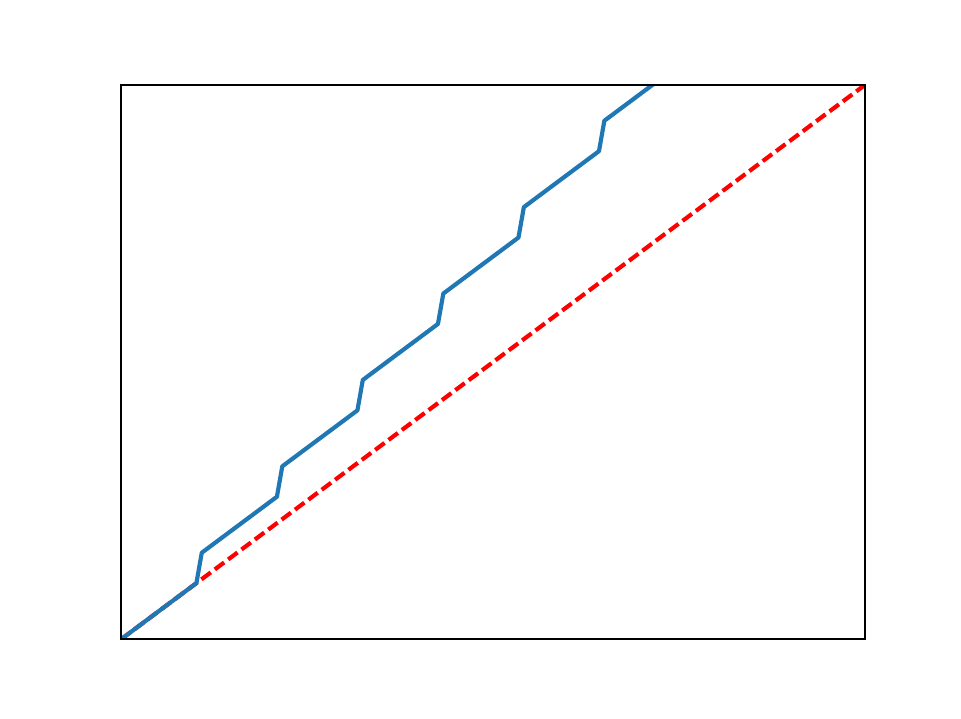}
  \caption{Piecewise linear}\label{fig:examples_B}
 \end{subfigure}
 \end{figure}

\begin{figure}[H]
 \ContinuedFloat
 \setlength{\x}{0.48\linewidth}
 \setlength{\y}{0.7\x}
 \begin{subfigure}[H]{\x}
  \centering
  \includegraphics[width = \y]{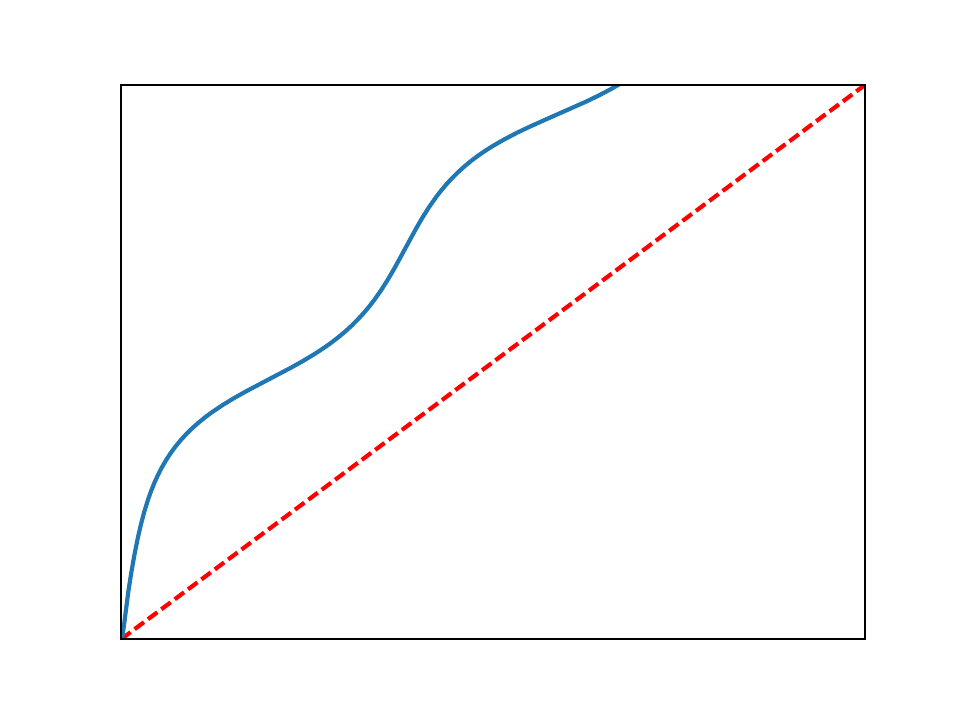}
  \caption{Differentiable}\label{fig:examples_C}
 \end{subfigure}
 \begin{subfigure}[H]{\x}
  \centering
  \includegraphics[width = \y]{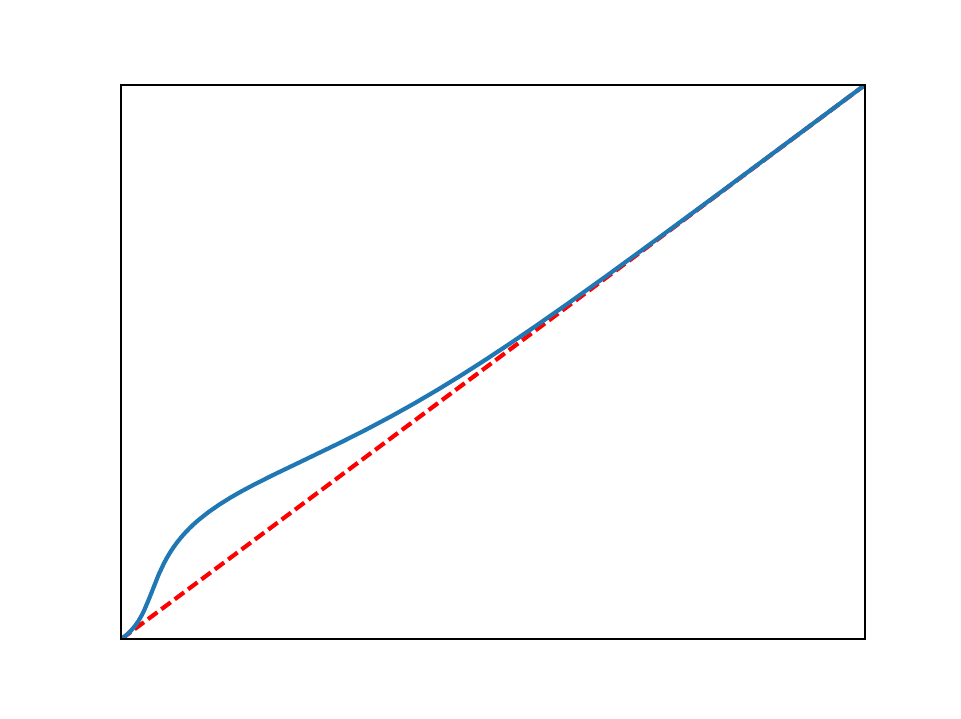}
  \caption{Differentiable}\label{fig:counterexample}
 \end{subfigure}
 \caption{Examples of expansive functions and one counterexample (solid blue lines). At each point, the function growths at least as fast as the line $y=x$ (dashed red line).}\label{fig:examples}
\end{figure}
\end{example}

\begin{remark} \label{rmk:noabs}
If $\eta$ is increasing, Eq.~\eqref{eqn:expansive} is equivalent to
\begin{equation}\label{eqn:expansivenoabs}
	\eta(x) - \eta(y) \ge x - y \quad \text{for all } x \ge y.
\end{equation}
In fact, this equation alone implies $\eta$ increasing and thus Eq.~\eqref{eqn:expansive}. We could drop the continuity hypothesis in Definition \ref{def:expansive}, and define an expansive function simply by Eq.~\eqref{eqn:expansivenoabs}. In practice, however, a monotonic transformation (Definition \ref{def:monotonic-transformation}) can always be realised by a continuous, piecewise linear function $\eta$ (Lemma \ref{lem:monotonic_without_eta}). 
\end{remark}
\subsection{Monotonic transformations}\label{sec:MonotonicTransformations}
In Kleinberg's original Consistency axiom, arbitrary transformations that increase inter-cluster distances and decrease intra-cluster distances are allowed. To avoid an impossibility result, we restrict to transformations obtained via an expansive function $\eta$, as follows. Recall that we write $x \sim_\mathcal{P} y$ if $x$ and $y$ are in the same cluster with respect to a partition $\mathcal{P}$, and $x \not\sim_\mathcal{P} y$ if not.

\begin{definition}\label{def:monotonic-transformation}
Let $d$ be a distance on a set $X$, and $\mathcal P$ a partition of $X$. A \emph{${\mathcal P}$-monotonic transformation of $d$} is any distance $d'$ on $X$ such that 
\begin{eqnarray}\label{eqn:MonotonicTransformation}
\begin{cases}
 d(x,y) = \eta(d'(x,y)) 	  &\text{ if $x \sim_\mathcal{P} y$, and}\\
 d(x,y) = \eta^{-1}(d'(x,y))  &\text{ if $x \not\sim_\mathcal{P} y$,}
\end{cases}
\end{eqnarray}
for some expansive map $\eta \colon [0,\infty) \to [0,\infty)$, and all $x, y \in X$. (Note that such $\eta$ necessarily satisfies $\eta(0)=0$.)
\end{definition}

\begin{definition}\label{def:MonotonicConsistent}
A clustering algorithm $F$ is \emph{Monotonic Consistent} if $F(d') = F(d)$ whenever $d'$ is a ${F(d)}$-monotonic transformation of $d$.
\end{definition}

Note that, given $d$ and $\mathcal{P}$, $d'$ is uniquely determined by $\eta$. Since $\eta(x) \ge x$ and $\eta^{-1}(x) \le x$ for all $x$ (Lemma \ref{lem:expansive-properties}), the distance function $d'$ increases inter-cluster distances and decreases intra-cluster distances (hence Consistency implies Monotonic Consistency). However, our allowed transformations do so globally ($d'$ depends on distances between points, not the actual points) and monotonically (the rates at which we expand or contract distances are the inverse of one another). Finally, note that $\mathcal{P}$-monotonic transformations can be composed and this corresponds to the composition $\eta_2 \circ \eta_1$ of expansive maps. 

\begin{example}
The following are examples of $\mathcal{P}$-monotonic transformations. 
\begin{enumerate}
\item (Linear) Let $\eta(x)=\alpha x$, $\alpha \ge 1$. The corresponding $\mathcal{P}$-monotonic transformation multiplies inter-cluster distances by $\alpha$, and intra-cluster distances by $1/\alpha$. This is similar to Inner and Outer Consistency, introduced in \cite{ackerman:2010}, except that the expansion and contraction rates are not arbitrary, but the reciprocal of one another.
\item (Linear step function) This is the function
\begin{equation}\label{eq:step_linear}
	\eta(x) = 
	\begin{cases}
	 x & 0 \le x \le d_1,\\
	 \alpha (x-d_1)+d_1 & d_1 \le x \le d_2, \\
	 (x-d_2) + \alpha d_2 & d_2 \le x,
	\end{cases}
\end{equation}
for some $0 \le d_1 < d_2$ and $\alpha>1$. The associated $\mathcal{P}$-monotonic transformation preserves (inter- or intra-cluster) distances below $d_1$, scales distances between $d_1$ and $d_2$ as in \Cref{fig:examples_B}, and (necessarily) translates distances above $d_2$, adding $\eta(d_2)=\alpha d_2$ to inter-cluster distances, and subtracting $\eta(d_2)$ to intra-cluster distances. Note that $d_2$ can be equal to $+\infty$ and so the third line in \Cref{eq:step_linear} becomes obsolete.
\item (Piecewise linear) This is generalises both (1) and (2): For the piecewise linear $\eta$ as in \Cref{fig:examples_C}, we have a rate of expansion/contraction $\alpha_i$, and a translation by $\eta(d_i)$, for distances in the interval $[d_i,d_{i+1}]$ where $\eta$ is linear. It can be shown that each piecewise linear function is a composition of linear step functions. 
\end{enumerate}
\end{example}
Below, we show that every $\mathcal{P}$-monotonic transformation is induced by a piecewise linear $\eta$, or, equivalently, by a finite composition of linear step functions. 

\subsection{Characterisation of monotonic transformations} \label{sec:Characterisation}
Although $d'$ is uniquely determined by $\eta$, this $\eta$ is not unique, that is, different choices of $\eta$ may result in the same $\mathcal{P}$-monotonic transformation $d'$. Indeed, any expansive $\eta$ interpolating the points $(d'(x,y),d(x,y))$ with $x \sim_\mathcal{P} y$ and $(d(x,y),d'(x,y))$ with $x \not\sim_\mathcal{P} y$ necessarily gives the same $\mathcal{P}$-monotonic transformation $d'$, by Eq.~\eqref{eqn:MonotonicTransformation}.  In particular, we can always assume $\eta$ to be piecewise linear in Definition \ref{def:monotonic-transformation}, and, in fact, we can determine whether such function exists directly from $d'$, as the next result shows. 

\begin{figure}[H]
 \setlength{\x}{0.48\linewidth}
 \setlength{\y}{0.8\x}
 \begin{subfigure}[c]{\x}
  \centering
  \includegraphics[width = \y]{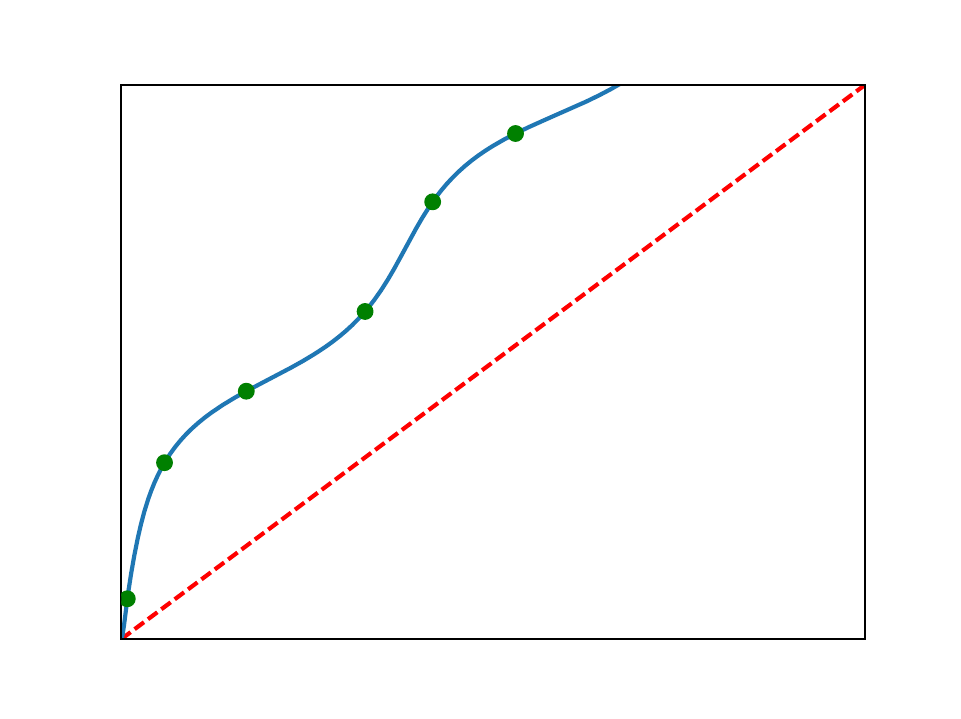}
 \end{subfigure}
 \begin{subfigure}[c]{\x}
  \centering
  \includegraphics[width = \y]{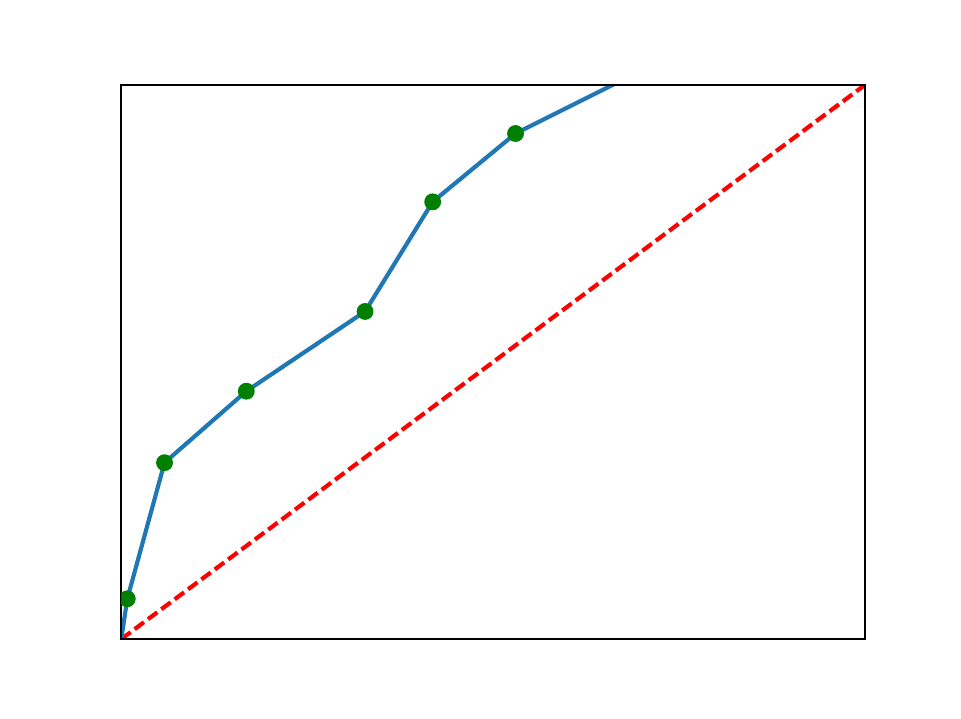}
 \end{subfigure}
 \caption{Expansive map (left) and linear interpolation (right) through the points in the subset $S$ (as in Lemma \ref{lem:monotonic_without_eta}). Both maps determine the same $\mathcal{P}$-monotonic transformation $d'$ of a distance $d$. In the linear interpolation (right), the slope of each successive segment must be at least 1.}
\end{figure}


\begin{lemma}\label{lem:monotonic_without_eta}
Let $d$ and $d'$ be distances on a finite set $X$ and $\mathcal{P}$ a partition of $X$. Then $d'$ is a $\mathcal{P}$-monotonic transformation of $d$ if and only if a linear interpolation of the points
\[
  S = \left\{ \left(d(x,y),d'(x,y)\right) \mid  x \sim_\mathcal{P} y \right\} \ \cup \ 
  \left\{ \left(d'(x,y),d(x,y)\right) \mid  x \not\sim_\mathcal{P} y \right\} \subseteq \R^2
\]
is a well-defined expansive map $\eta \colon [0,\infty) \to [0,\infty)$.
\end{lemma}
\begin{proof}
Clearly, if there exists a linear interpolation $\eta$ of the points in $S$
such that it is a well-defined expansive map, then $d'$ is a $\mathcal{P}$-monotonic transformation of $d$, by definition. 

Now assume $d'$ is a $\mathcal{P}$-monotonic transformation of $d$. Then we can write
\[
  S=\left\{ \left(d(x,y),\eta(d(x,y))\right) \mid  x \sim_\mathcal{P} y \right\} \ \cup \
  \left\{ \left(d'(x,y),\eta(d'(x,y))\right) \mid  x \not\sim_\mathcal{P} y \right\},
\]
where $\eta \colon [0,\infty) \to [0,\infty)$ is an expansive map. 
To define a linear interpolation of $S$ we will assume that $S$ is ordered lexicographically
\[
  S=\{(x_0, y_0),(x_1,y_1),\dots,(x_N,y_N)\},
\]
where $y_i = \eta(x_i)$ for $0\leq i\leq N$ and $x_i<x_{i+1}$. We can assume the latter since $\eta$ is injective: if $x_i = x_{i+1}$ then $y_i=y_{i+1}$. Consider now the linear interpolation of $S$ consisting of segments between consecutive pairs of points $(x_i, y_i)$ and $(x_{i+1},y_{i+1})$. As every point in $S$ is of the form $(x,\eta(x))$, we have that the slope of each segment is
\[
 \frac{\eta(x_{i+1})-x_{i+1}}{\eta(x_{i})-x_{i}}\geq 1,
\]
as $\eta$ is expansive, Eq.~\eqref{eqn:expansivenoabs}. From this we have that the linear interpolation above, effectively a discretization of $\eta$, is in fact a well-defined expansive map.
\end{proof}

\subsection{Avoidance of problematic behaviour} \label{sec:AvoidanceProblematic}
Recall that Kleinberg's Consistency axiom allows us to arbitrarily emphasize any subcluster structure within a cluster without changing the output of the clustering algorithm (Fig.~\ref{fig:1}). We now show how this problematic behaviour is explicitly avoided by Monotonic Consistency. Suppose that we have a set $X$ and a partition $\mathcal{P}=F(d)$ with respect to a clustering algorithm $F$ and a distance $d$ on $X$. Choose a cluster $C$ and a partition $C=C_1 \cup C_2$ that we wish to emphasize on a new distance $d'$ which (necessarily) decreases the intra-cluster distances, but in a way that distances within each $C_1$ and $C_2$ decrease much faster than distances between  $C_1$ and $C_2$, in order to achieve the behaviour depicted in Fig.~\ref{fig:1}. 

Let $u, v \in C_1$ distinct and $w \in C_2$, and call $x=d(u,v)$, $x'=d'(u,v)$, $y=d(u,w)$ and $y'=d'(u,w)$. We impose $x' \le x$ and $y' \le y$, and, in addition, we want to make $x - x'$ large while keeping $y - y'$ small (Fig.~\ref{fig:avoidance}). This is not possible if $d'$ if a $\mathcal{P}$-monotonic transformation of $d$, as follows. Let $\eta$ be an expansive map realising $d'$. Then $x=\eta(x')$ and $y=\eta(y')$. Assume first $x \le y$. Then Eq.~\eqref{eqn:expansivenoabs} gives
\begin{align}\label{eq:7}
	\eta(y') - y' \ge \eta(x') - x' \iff y-y' \ge x-x'.
\end{align}
This implies that if we want to reduce the distances inside of a subcluster ($x-x'$ large), we need to reduce the distances between the clusters ($y-y'$) by at least the same amount. The remaining case, $x \ge y$, follows from $\eta^{-1}$ being a decreasing function (Lemma \ref{lem:expansive-properties}),
\begin{align}\label{eq:8}
	x \ge y \implies x' = \eta^{-1}(x) \ge \eta^{-1}(y) = y',
\end{align}
so that we cannot decrease the intra-cluster distance $x$ without also decreasing the inter-cluster distance $y$. 

\begin{figure}[H]
 \centering
 \includegraphics[width = .4\textwidth]{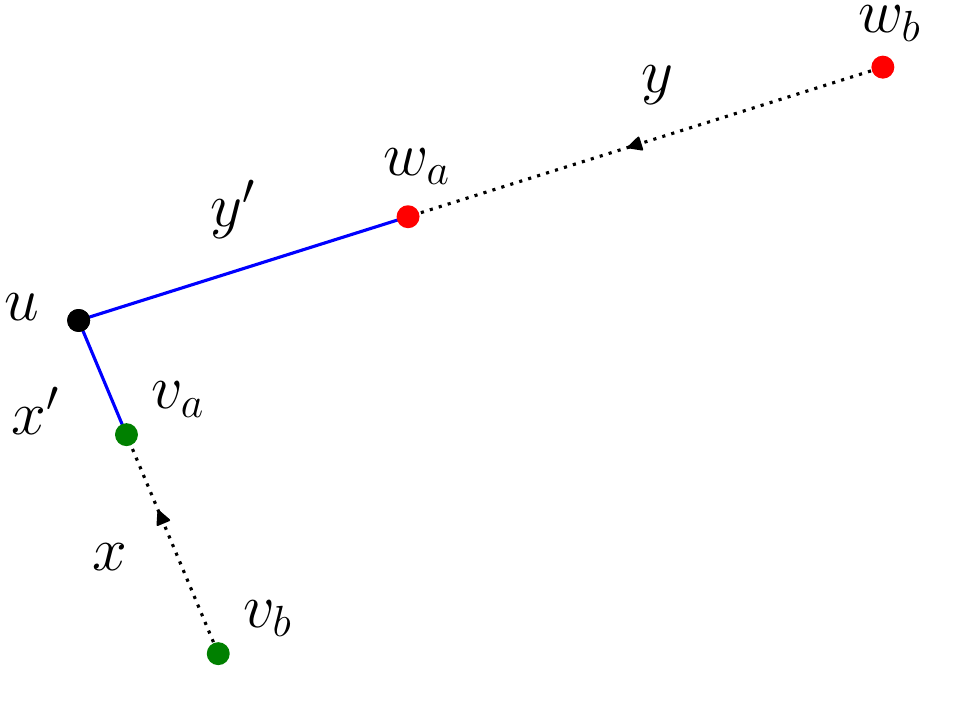}
 \caption{Avoidance of the problematic behaviour by Monotonic Consistency. A $\mathcal{P}$-monotonic transformation of $d$ reduces the distance from $u$ to $v$ by $x-x'$, and the distance from $u$ to $w$ by $y-y'$ (here shown with subscripts `b' and `a' indicating `before' and `after' the transformation). Then either $x' \ge y'$ (Eq.~\eqref{eq:8}), or $y-y' \ge x-x'$ (Eq.~\eqref{eq:7}). In either case, we cannot separate $u$ and $v$ from $w$ within the same cluster.}
 \label{fig:avoidance}
\end{figure}

We finish Section \ref{sec:MonotonicConsistency} by exploring Monotonic Consistency for Single Linkage, and for metrics. 

\subsection{Single Linkage does not satisfy Monotonic Consistency}
We will show that Monotonic-Consistency, a weakening of Consistency, can be satisfied together with Richness and Scale-Invariance by a particular instance of Morse clustering (Corollary \ref{cor:Possibility}). This is in contrast with Single Linkage, which, with different stopping conditions, satisfies each pair of Kleinberg's axioms \cite{kleinberg:2003}. The instance of Single-Linkage satisfying Richness and Scale Invariance, namely Scale-$\alpha$ Single Linkage with $0 < \alpha < 1$, does not satisfy Monotonic Consistency, as we show next. This means that no version of Single Linkage can be used to show our possibility result for Monotonic Consistency. Recall that Scale-$\alpha$ Single Linkage applied to $(X,d)$ returns the connected components of the graph with vertex $X$ and edges $(i,j)$ such that $d(i,j)<\alpha \cdot \max_{s,t \in X}d(s,t)$.
\begin{lemma}
Let $\alpha \in (0,1)$. Then Scale-$\alpha$ Single-Linkage does not satisfy Monotonic Consistency.
\end{lemma}
\begin{proof}
Let $X$ be any set with at least three points, $\mathcal{P}$ any partition of $X$ with at least two clusters, and $x,y \in X$ such that $x\not\sim_\mathcal{P} y$. Define $d$ on $X$ as follows
\[
	d(u,v) =
\begin{cases}
 	\dfrac{\alpha}{2}, 	& \text{ if } u \sim_\mathcal{P} v,\\
 	1,					& \text{ if } u=x, v=y,\\
 	\alpha, 			& \text{ otherwise}.
\end{cases}
\]
Let $d_\text{max} = \max_{s,t \in X}d(s,t) = 1$. If we represent $(X,d)$ by a complete graph with vertex set $X$ and edges $(i,j)$, $i\neq j$, weighted by $d(i,j)>0$, Scale-$\alpha$ Single-Linkage returns the connected component of the graph obtained after removing all edges $(i,j)$ with value $d(i,j)\ge\alpha \, d_\text{max} = \alpha$, in this case. Consequently, Scale-$\alpha$ Single-Linkage applied to $d$ returns the original partition ${\mathcal P}$.  

Let $d'$ be the ${\mathcal P}$-monotonic transformation of $d$ given by 
$$
	\eta(x) = \frac{x^2+x}{\alpha}.
$$
(Note that $\eta(0)=0$ and $\eta'(x)=\frac{2x+1}{\alpha}>1$ for all $x$, so $\eta$ is indeed expansive.) Then
\[
	d'(u,v) =
\begin{cases}
 	\eta^{-1}\left(\frac{\alpha}{2}\right) = \frac{-1+\sqrt{1+2\alpha^2}}{2}, 	
    & \text{ if } u \sim_\mathcal{P} v,\\
 	\eta(1) = \frac{2}{\alpha},		
    & \text{ if } u=x, v=y,\\
 	\eta(\alpha) = 1+\alpha, 			& \text{ otherwise}.
\end{cases}
\]
We now have $d'_\text{max}= \eta(1) = \frac{2}{\alpha}$ and thus scale$-\alpha$ Single-Linkage removes the edges $(i,j)$ with $d(i,j) \ge \alpha d'_\text{max} = 2$. Since $\alpha < 1$, the only removed edge is $d(x,y)$ and, since $X$ has at least three points, the algorithm returns the trivial partition $\{X\}$, clearly not $\mathcal{P}$.
\end{proof}


\subsection{Monotonic Consistency for metrics}
A \emph{metric} is a distance (in the sense of this article) which also satisfies the triangle inequality, $d(u,w) \le d(u,v) + d(v,w)$ for all $u,v,w$. Metrics arise naturally when $X$ is embedded in a metric space such as $\mathbb{R}^m$, and, in fact, for many clustering algorithms (for example $k$-means clustering), the distance function is always a metric. It is therefore natural to ask whether Monotonic Consistency is a useful property in this context, namely, whether a non-trivial (that is, $\eta$ not the identity) ${\mathcal P}$-monotonic transformation of a metric can be a metric. (If not, Monotonic Consistency would become an empty clustering axiom for metrics.) Of course, not every ${\mathcal P}$-monotonic transformation of a metric will be a metric, but we show below that, given a metric $d$ and an arbitrary partition ${\mathcal P}$, we can always find ${\mathcal P}$-monotonic transformations of $d$ which are metrics.

Given a distance $d$ on a set $X$, we call a triple of points $i,j,k \in X$ \emph{aligned} if they are distinct and $d(i,k) = d(i,j) + d(j,k)$. 

\begin{theorem}\label{thm:monotonic-metric}
Let $X$ be a set, $\mathcal P$ a partition of $X$, and $d$ a distance on $X$ such that no triple of nodes is aligned. Then there exists a constant $c(d,\mathcal{P})>1$ such that, for all $s \in [1, c(d,\mathcal{P}))$, the ${\mathcal P}$-monotonic transformation of $d$ given by $\eta(x)=sx$ is a metric. Moreover, there is a universal constant $c(d)$ independent of the partition, that is, $1<c(d)\le c(d,\mathcal{P})$ for all partitions $\mathcal P$ of $X$. 
\end{theorem}
\begin{proof}
Let $d'$ be the ${\mathcal P}$-monotonic transformation of $d$ given by $\eta(x)=sx$ for some $s \ge 1$. We will find conditions on $s$ to guarantee that $d'$ satisfies the triangle inequality. Let $i,j,k \in X$ distinct (if not, the triangle inequality is automatically satisfied). We want to show that
\[
  d(i,k) \le d(i,j) + d(j,k)\implies d'(i,k) \le d'(i,j) + d'(j,k).
\]
Recall that
\[
d'(i,j) = 
 \begin{cases}
  s\,d(i,j)&\text{ if }i \sim_\mathcal{P} j,\\
  \dfrac{d(i,j)}{s}&\text{ otherwise.}
 \end{cases}
\]
If $i$, $j$ and $k$ are in the same cluster then clearly
\[\frac{d(i,k)}{s} \le \frac{d(i,j)}{s}+\frac{d(j,k)}{s}.\]
If they are all in pairwise different clusters, then
\[s\ d(i,k)\le s\ d(i,j)+s d(j,k).\]
If $i$ and $k$ are in the same cluster but $j$ is not, then (recall $s \ge 1$)
\[\frac{d(i,k)}{s} \le d(i,k) \le s\, d(i,j)+s\, d(j,k).\]
Since $i$ and $k$ are interchangeable in the triangle inequality above, the only remaining case is when $i$ and $j$ are in the same cluster, but $k$ is not. In this case, we want to show that 
\begin{equation}\label{eqn:trian-ineq}
	s\, d(i,k) \le \frac{d(i,j)}{s} + s\, d(j,k).
\end{equation}
If $d(i,k) \le d(j,k)$ then $s\, d(i,k) \le s\, d(j,k)$ and Eq.~\eqref{eqn:trian-ineq} is automatically satisfied. If $d(i,k) > d(j,k)$, Eq.~\eqref{eqn:trian-ineq} is satisfied if and only if 
\[
	s^2 \left( d(i,k) - d(j,k) \right) \le d(i,j)
    \iff s \le \sqrt{\frac{d(i,j)}{d(i,k)-d(j,k)}}\,.
\]
Define 
\begin{align*}
	c(d,\mathcal{P}) &=  
    	\min_{\substack{
        i\sim_\mathcal{P} j, 
        i\not\sim_\mathcal{P} k\\
        d(i,k) > d(j,k)}}
        \sqrt{\frac{d(i,j)}{d(i,k)-d(j,k)}}
        & \text{and}  &
    & c(d) = \min_{\substack{
        d(i,k) > d(j,k)}}
        \sqrt{\frac{d(i,j)}{d(i,k)-d(j,k)}}\,.
\end{align*}
Clearly, $c(d) \le c(d,\mathcal{P})$ for all partitions $\mathcal{P}$. To finish the proof, note that the triangle inequality for $d$ guarantees $c(d) \ge 1$, and $c(d) = 1$ if and only if there is an aligned triple of points.
\end{proof}

Defining the minimum of an empty set as infinity, we might have $c(d,\mathcal{P})=\infty$ (or $c(d)=\infty$), meaning that the $\mathcal{P}$-monotonic transformation of $d$ given by $\eta(x)=sx$ is a metric for any $s \ge 1$, and Theorem \ref{thm:monotonic-metric} still holds. Of course, this would only occur if for all $i$, $j$, $k$ with $i\sim_\mathcal{P} j$ and $i\not\sim_\mathcal{P} k$, we have $d(i,k)=d(j,k)$. 


\section{Morse Clustering} \label{sec:MorseClustering}
In this section we consider the clustering algorithm \texttt{Morse} in the form of three variants: \texttt{SiR-Morse}, $k$-\texttt{Morse} and $\delta$-\texttt{Morse} (the last two are described in full detail in \ref{appendix:instances}). 
Each of them satisfy one pair of the original Kleinberg axioms, and all of them satisfy Monotonic Consistency. In particular, one of them (\texttt{SiR-Morse}) satisfies Scale Invariance and Richness, showing that our three axioms can be simultaneously satisfied (Corollary \ref{cor:Possibility}). Morse clustering is inspired by Topology and Differential Geometry, namely Morse theory \cite{Milnor:1963} and its discretisation due to Forman \cite{forman:1998}. We start with a brief introduction to both continuous and discrete Morse theory and explain how they motivate our clustering algorithm. 

\begin{figure}
 \centering
\begin{center}
 \begin{tikzpicture}[scale = 0.35]
  \draw (0,7) node {\textit{A}};
  \draw (0,1) node {\textit{B}};
  \draw (0,-1) node {\textit{C}};
  \draw (0,-7) node {\textit{D}};
  \draw (-9,0) node {height};
  \draw[line width=0.3mm,decoration={markings, mark=at position 1 with {\arrow{>}}},postaction={decorate}] (-7,6.5) -- (-7, -6.5);

  \draw[line width=0.1mm] (0,0) circle (6cm);
  \draw[line width=0.1mm] (0,0) circle (2cm);
  \draw[line width=0.2mm] (0,6) arc (90:270:1cm and 2cm);
  \draw[line width=0.2mm] (0,-2) arc (90:270:1cm and 2cm);
  \draw[line width=0.2mm,dashed] (0,-6) arc (-90:90:1cm and 2cm);
  \draw[line width=0.2mm,dashed] (0,2) arc (-90:90:1cm and 2cm);
  \draw[line width=0.2mm] (-2,0) arc (0:-180:2cm and 1cm);
  \draw[line width=0.2mm] (6,0) arc (0:-180:2cm and 1cm);

  \filldraw[fill=red, draw=red] (0,6) circle (.2);
  \filldraw[fill=red, draw=red] (0,2) circle (.2);
  \filldraw[fill=red, draw=red] (0,-2) circle (.2);
  \filldraw[fill=red, draw=red] (0,-6) circle (.2);

  \draw[blue,line width=0.3mm,decoration={markings, mark=at position 1 with {\arrow{>}}},postaction={decorate}] (3,5.2) arc (60:30:6cm);
  \draw[blue,line width=0.3mm,decoration={markings, mark=at position 1 with {\arrow{>}}},postaction={decorate}] (-3,5.2) arc (120:150:6cm);
  \draw[blue,line width=0.3mm,decoration={markings, mark=at position 1 with {\arrow{>}}},postaction={decorate}] (5.2,-3) arc (-30:-60:6cm);
  \draw[blue,line width=0.3mm,decoration={markings, mark=at position 1 with {\arrow{>}}},postaction={decorate}] (-5.2,-3) arc (-150:-120:6cm);

  \draw[blue,line width=0.3mm,decoration={markings, mark=at position 1 with {\arrow{>}}},postaction={decorate}] (1,1.73) arc (60:30:2cm);
  \draw[blue,line width=0.3mm,decoration={markings, mark=at position 1 with {\arrow{>}}},postaction={decorate}] (-1,1.73) arc (120:150:2cm);
  \draw[blue,line width=0.3mm,decoration={markings, mark=at position 1 with {\arrow{>}}},postaction={decorate}] (1.73,-1) arc (-30:-60:2cm);
  \draw[blue,line width=0.3mm,decoration={markings, mark=at position 1 with {\arrow{>}}},postaction={decorate}] (-1.73,-1) arc (-150:-120:2cm);

  \draw[blue,line width=0.3mm,decoration={markings, mark=at position 1 with {\arrow{>}}},postaction={decorate}] (.5,5.73) arc (60:30:1cm and 2cm);
  \draw[blue,line width=0.3mm,decoration={markings, mark=at position 1 with {\arrow{>}}},postaction={decorate}] (-.5,5.73) arc (120:150:1cm and 2cm);
  \draw[blue,line width=0.3mm,decoration={markings, mark=at position 1 with {\arrow{>}}},postaction={decorate}] (.97,3.5) arc (-15:-45:1cm and 2cm);
  \draw[blue,line width=0.3mm,decoration={markings, mark=at position 1 with {\arrow{>}}},postaction={decorate}] (-.97,3.5) arc (195:225:1cm and 2cm);

  \draw[blue,line width=0.3mm,decoration={markings, mark=at position 1 with {\arrow{>}}},postaction={decorate}] (.5,-2.27) arc (60:30:1cm and 2cm);
  \draw[blue,line width=0.3mm,decoration={markings, mark=at position 1 with {\arrow{>}}},postaction={decorate}] (-.5,-2.27) arc (120:150:1cm and 2cm);
  \draw[blue,line width=0.3mm,decoration={markings, mark=at position 1 with {\arrow{>}}},postaction={decorate}] (.97,-4.5) arc (-15:-45:1cm and 2cm);
  \draw[blue,line width=0.3mm,decoration={markings, mark=at position 1 with {\arrow{>}}},postaction={decorate}] (-.97,-4.5) arc (195:225:1cm and 2cm);

  \draw[blue,line width=0.3mm,decoration={markings, mark=at position 1 with {\arrow{>}}},postaction={decorate}] (4,-.97) arc (0:-20:4cm and 5cm);
  \draw[blue,line width=0.3mm,decoration={markings, mark=at position 1 with {\arrow{>}}},postaction={decorate}] (-4,-.97) arc (180:200:4cm and 5cm);

  \filldraw[blue] (3,5.2) circle (0.1cm);
  \filldraw[blue] (-3,5.2) circle (0.1cm);
  \filldraw[blue] (5.2,-3) circle (0.1cm);
  \filldraw[blue] (-5.2,-3) circle (0.1cm);

  \filldraw[blue] (1,1.73) circle (0.1cm);
  \filldraw[blue] (-1,1.73) circle (0.1cm);
  \filldraw[blue] (1.73,-1) circle (0.1cm);
  \filldraw[blue] (-1.73,-1) circle (0.1cm);

  \filldraw[blue] (.5,5.73) circle (0.1cm);
  \filldraw[blue] (-.5,5.73) circle (0.1cm);
  \filldraw[blue] (.97,3.5) circle (0.1cm);
  \filldraw[blue] (-.97,3.5) circle (0.1cm);

  \filldraw[blue] (.5,-2.27) circle (0.1cm);
  \filldraw[blue] (-.5,-2.27) circle (0.1cm);
  \filldraw[blue] (.97,-4.5) circle (0.1cm);
  \filldraw[blue] (-.97,-4.5) circle (0.1cm);

  \filldraw[blue] (4,-.97) circle (0.1cm);
  \filldraw[blue] (-4,-.97) circle (0.1cm);
 \end{tikzpicture}
\end{center}
 \caption{Morse function (vertical height) on a torus, critical points (red), and associated flow (blue). The flow represents a unique maximal descent (or ascend, if we reverse time) path that a particle, such as a drop of water, would follow on the surface. It is defined everywhere except at four critical points, which can be thought of as a flow source ($A$), sink ($D$) or a combination of both ($B$, $C$). The  number and type of critical points, for any Morse function, is a topological invariant of the torus.}\label{fig:2:1}
\end{figure}
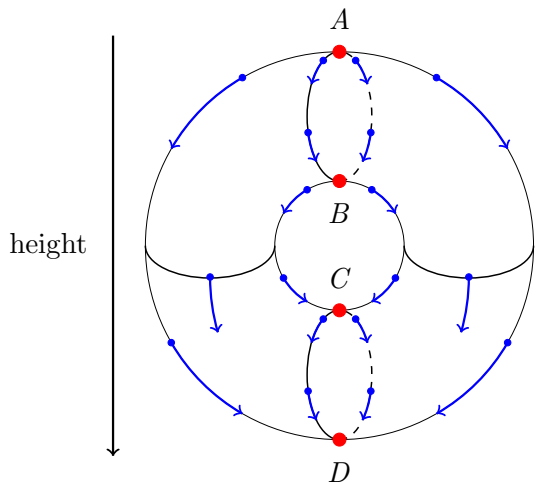

\subsection{Morse theory}\label{sec:MorseTheory}
Topology is the mathematical study of `shape' \cite{prasolov:1995}. It considers properties of a space (such as a 2D surface, or 3D object) which are invariant
under continuous deformations such as stretching, bending or collapsing. A topological invariant is a property, for example whether the space is disconnected, which is invariant under such deformations. A standard approach in Topology is to study a space via functions defined on the space. Morse theory \cite{Milnor:1963} considers potential-like functions called Morse functions and their associated flow on the space, defined by the unique direction of maximal descent at every point, except at a few critical points (see Fig.~\ref{fig:2:1}).

\cite{forman:1998} introduced a discrete version of Morse Theory which applies to discretisations of continuous spaces, such as a polygonal mesh of a continuous surface. Such discretisation decomposes the space into vertices, edges, triangles, etc.~called simplices. A discrete Morse function assigns a real number to each simplex under certain combinatorial restrictions, and we have associated notions of critical simplex, and discrete Morse flow (Fig.~\ref{fig:2:5}). 

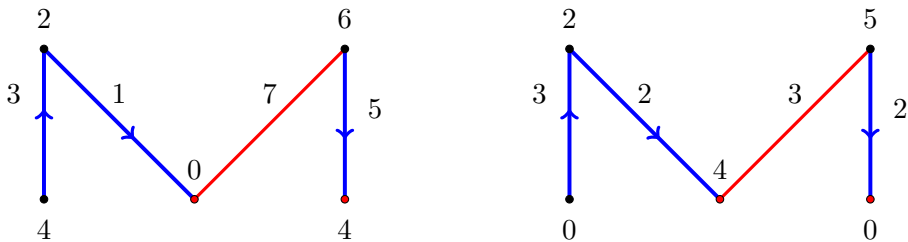
\begin{figure}
 \centering
\begin{minipage}{0.4\textwidth}
 \centering
 \begin{tikzpicture}
    \draw (2,.4) node {\textit{$0$}};
    \draw (1.,1.4) node {\textit{$1$}};
    \draw (0,2.4) node {\textit{$2$}};
    \draw (-.4,1.4) node {\textit{$3$}};
    \draw (4,-.4) node {\textit{$4$}};
    \draw (0,-.4) node {\textit{$4$}};
    \draw (4.4,1.2) node {\textit{$5$}};
    \draw (4,2.4) node {\textit{$6$}};
    \draw (3.,1.4) node {\textit{$7$}};
    
    \draw[red, very thick] (2,0) -- (4,2);
    \draw[blue,line width = 1.5, decoration={markings, mark=at position 0.6 with {\arrow{>}}},postaction={decorate}] (0,0) -- (0,2);
    \draw[blue,line width = 1.5, decoration={markings, mark=at position 0.6 with {\arrow{>}}},postaction={decorate}] (0,2) -- (2,0);
    \draw[blue,line width = 1.5, decoration={markings, mark=at position 0.6 with {\arrow{>}}},postaction={decorate}] (4,2) -- (4,0);
    
    \filldraw[fill = red, draw = black] (2,0) circle (.05);
    \filldraw[fill = red, draw = black] (4,0) circle (.05);
    \filldraw[fill = black, draw = black] (0,0) circle (.05);
    \filldraw[fill = black, draw = black] (0,2) circle (.05);
    \filldraw[fill = black, draw = black] (4,2) circle (.05);
  \end{tikzpicture}
\end{minipage}
\begin{minipage}{0.45\textwidth}
 \centering
 \begin{tikzpicture}
    \draw (2,.4) node {\textit{$4$}};
    \draw (1.,1.4) node {\textit{$2$}};
    \draw (0,2.4) node {\textit{$2$}};
    \draw (-.4,1.4) node {\textit{$3$}};
    \draw (4,-.4) node {\textit{$0$}};
    
    \draw (0,-.4) node {\textit{$0$}};
    \draw (4.4,1.2) node {\textit{$2$}};
    \draw (4,2.4) node {\textit{$5$}};
    \draw (3.,1.4) node {\textit{$3$}};
    
    \draw[red, very thick] (2,0) -- (4,2);
    \draw[blue,line width = 1.5, decoration={markings, mark=at position 0.6 with {\arrow{>}}},postaction={decorate}] (0,0) -- (0,2);
    \draw[blue,line width = 1.5, decoration={markings, mark=at position 0.6 with {\arrow{>}}},postaction={decorate}] (0,2) -- (2,0);
    \draw[blue,line width = 1.5, decoration={markings, mark=at position 0.6 with {\arrow{>}}},postaction={decorate}] (4,2) -- (4,0);   
    
    \filldraw[fill = red, draw = black] (2,0) circle (.05);
    \filldraw[fill = red, draw = black] (4,0) circle (.05);
    \filldraw[fill = black, draw = black] (0,0) circle (.05);
    \filldraw[fill = black, draw = black] (0,2) circle (.05);
    \filldraw[fill = black, draw = black] (4,2) circle (.05);
  \end{tikzpicture}
\end{minipage}
 \caption{Morse flows on the same discrete space (here a small graph) associated to a discrete Morse function (left) and to arbitrary edge and vertex weight (right). Our algorithm (Algorithm \ref{alg:morseflow}) generalises the standard construction \cite{forman:1998} of the Morse flow associated to a discrete Morse function (left) to the Morse flow of an arbitrary edge and vertex weighted on a graph (right). In both cases, we show the associated discrete Morse flow as blue directed edges and critical simplices (vertices and edges) in red. A particle on a vertex has a unique direction of descent following the blue arrow, except at the two critical vertices shown in red, both sinks of the flow. After removing the critical edge, we have two connected components, each a tree rooted at a critical vertex (Algorithm \ref{alg:morseclustering}).} 
 \label{fig:2:5}
\end{figure}

Discrete Morse theory can be applied to clustering by representing a set $X$ with distance $d$ as an undirected weighted graph $G$ with vertex set $X$, and an edge between $i$ and $j$ if $d(i,j)>0$, and no such edge otherwise. (This is an all-to-all, or complete, graph.) A graph is a discretisation of a curve and hence discrete Morse theory applies. To obtain a partition of $X$ using Morse theory, first we extend the edge weights given by the distances $d(i,j)>0$ to a Morse function on the graph by assigning weights to the vertices as well. This Morse function determines a unique flow on the vertices which, in turn, gives a natural partition of the vertex set. The clusters are the connected components of the graph after removing the critical edges (edges not participating in the flow), and each cluster becomes a tree rooted at a critical vertex (a sink of the flow), see Fig.~\ref{fig:2:5}. We describe this in detail next.

\subsection{Morse flow} \label{sec:MorseFlowAlgorithm}
Let $X$ be a finite set and $d$ a distance (dissimilarity) on $X$. The Morse clustering of $(X,d)$ is obtained from the Morse flow on the graph representation of $(X,d)$, by removing the edges not participating in the flow. In turn, the Morse flow is determined by the direction of maximal descent at every vertex together with the initial and final vertex weights (to guarantee a descending path). In its more general form, rather than weights, we only need a way of comparing vertices and edges locally. Formally, this consists on a choice of vertex and edge preorders. 

A \emph{preorder} on a set is a binary relation $\preceq$ that is reflexive ($a \preceq a$ for all $a$) and transitive ($a \preceq b$ and $b \preceq c$ implies $a \preceq c$ for all $a, b, c$). We write $a \prec b$ if $a \preceq b$ and $b \not\preceq a$ (that is, $b \preceq a$ does not hold). A preorder is \emph{total} if $a \preceq b$ or $b \preceq a$ for all $a,b$. Our main examples are the total preorders induced by an edge or vertex weight function on a graph (Example \ref{exa:WeightsPreorders}). By a \emph{graph} $\mathcal{G}=(V,E)$ we mean a non-empty vertex set $V$ and an edge set $E \subseteq V \times V$ so that $(u,v) \in E$ represents a directed edge from $u$ to $v$. A graph is \emph{undirected} if $(v,w) \in E$ whenever $(w,v) \in E$, for all $v, w \in V$, \emph{loopless} if $(v,v) \not\in E$ for all $v \in V$, and \emph{finite} if $V$ (and therefore $E$) is a finite set. 


\begin{example}\label{exa:WeightsPreorders}
Let $\mathcal{G}=(V,E)$ be a graph.  
\begin{itemize}
\item[(1)] (Edge weights) For any function $w \colon E \to \R$, the relation $e \preceq f$ if $w(e) \le w(f)$ is a total preorder on $E$.
\item[(2)] (Vertex weights) For any function $w \colon V \to \R$, the relation $u \preceq v$ if $w(u) \le w(v)$ is a total preorder on $V$.
\end{itemize}
\end{example}
A distance $d$ on a set $X$ is an edge weight function for the complete graph with vertex set $X$, and hence induces a total edge preorder on the graph representation of $X$. Similarly, a labelling $X=\{x_1,x_2,\ldots,x_n\}$ induces a vertex weight $w(x_i)=i$ and hence a total preorder on the vertices $V=X$ of such graph representation. 

\begin{remark}
A preorder is an \emph{order} if it is also anti-symmetric ($a \preceq b$ and $b \preceq a$ implies $a=b$). Our examples above are not necessarily orders, as we may have $w(a)=w(b)$ with $a\neq b$. If $\preceq$ is a total order, $a \prec b$ is equivalent to $a \preceq b$ and $a \neq b$. Note that  any total preorder on a set is induced by a weight function $w \colon X \to \mathbb{N}$. 
\end{remark}



Morse clustering applies to an arbitrary finite graph $\mathcal{G}$ with a choice of edge and vertex preorders $\preceq_E$ and $\preceq_V$. First, it finds the (ascending) Morse flow associated to $(\mathcal{G}, \preceq_E, \preceq_V)$ (Algorithm \ref{alg:morseflow}), then the vertex partition associated to the Morse flow, that is, the connected components of the graph after removing the critical edges (Algorithm \ref{alg:morseclustering}); see also Fig.~\ref{fig:2:5}. First, we need to introduce some notation and terminology. 

Given a node $v$, we define the set of \emph{edges at $v$} as
\[
	E_v = \left\{  (v,w) \in E \right\}.
\]
A \emph{maximal edge at $v$} is a maximum for $E_v$ with respect to the edge preorder, that is, an edge $e \in E_v$ such that $f \preceq e$ for all $f \in E_v$. Note that a maximal edge at $v$ may not exist (e.g.~if the preorder is not total), or it may not be unique (e.g.~if $w(e)=w(f)$ for some edge weights). An edge $(v,w)$ is \emph{ascending}, or \emph{admissible}, if $v \prec_V w$. If an edge is the unique maximal edge at $v$, and it is ascending, we call it a \emph{non-critical} or \emph{flow} edge at $v$. 

We define the \emph{Morse flow} of a graph $\mathcal{G}$ with a choice of preorders $\preceq_E$ and $\preceq_V$ as the map $\Phi \colon V \to V$ given by
\begin{align*}
	\Phi(v) = \begin{cases} 
    	w & \text{if $(v,w)$ is the unique maximal edge at $v$, and it is ascending $v \prec_V w$,}\\
	v & \text{otherwise.}
	\end{cases}.
\end{align*}

\begin{remark}
There is a similar notion of \emph{descending} edges and Morse flow. For simplicity, we define Morse flow as ascending, and achieve descending flows simply by reversing the vertex preorder $\preceq_V$.
\end{remark}

Graphically, we can represent a Morse flow as directed edges $(v,w)$ whenever $\Phi(v)=w$ (blue directed edges in Fig.~\ref{fig:2:5}). Edges not participating in the flow (red edges in Fig.~\ref{fig:2:5}) are called \emph{critical edges}, and fixed points of the flow, $\Phi(v)=v$, are called \emph{critical vertices} (red vertices in Fig.~\ref{fig:2:5}), and correspond to `sinks' of the flow. 

\begin{algorithm2e}
\SetAlgoLined
\DontPrintSemicolon
\SetKwFunction{Sum}{sum}\SetKwFunction{Diag}{diag}
\KwIn{graph $\mathcal{G}=(V,E)$, edge preorder $\preceq_E$, vertex preorder $\preceq_V$}
\KwOut{Morse flow $\Phi \colon V \to V$}
\BlankLine
\ForEach{$v \in V$}{\eIf{maximal edge $e_v=(v,w)$ of $E_v$ exists and it is unique and $v \prec_V w$}{$\Phi(v)=w$\;}{$\Phi(v)=v$\;}}
\caption{Morse flow algorithm.}
\label{alg:morseflow}
\end{algorithm2e}

The Morse flow of a graph can be computed using Algorithm \ref{alg:morseflow}. This algorithm can be easily implemented so that each edge is visited only once, and thus has linear time complexity $\mathcal{O}(m)$ where $m$ is the number of edges.

An important observation is that we \emph{first} use the edge preorder to extract the maximal edge at a vertex (if it exists and is unique), and \emph{then} use it in the flow only if it is also ascending. In particular, if the maximal edge does not exist, or it is not unique, or, crucially, it is not ascending, we define $\Phi(v)=v$, that is, the flow stops at $v$. 
This design choice works well in practice \cite{schofield:2019}, produces a rich family of clustering algorithms (Section \ref{sec:ThreeInstances} and \ref{appendix:instances}) and, crucially, allows us to distinguish local maxima (Fig.~\ref{fig:local_maxima}) without introducing additional scaling/threshold parameters. 


Note that the Morse flow algorithm work for both directed and undirected networks. For undirected networks (the case we are concern with in this paper), each undirected edge $\{(v,w), (w,v)\}$ is evaluated twice, once at $v$ and once at $w$. Since at most one of $(v,w)$ or $(w,v)$ is admissible (ascending), at most one of them belongs to the flow. The fact that the flow is (strictly) ascending, also means that there cannot be any cycles in the flow. 



\begin{figure}
 \centering
 \includegraphics[scale=0.4]{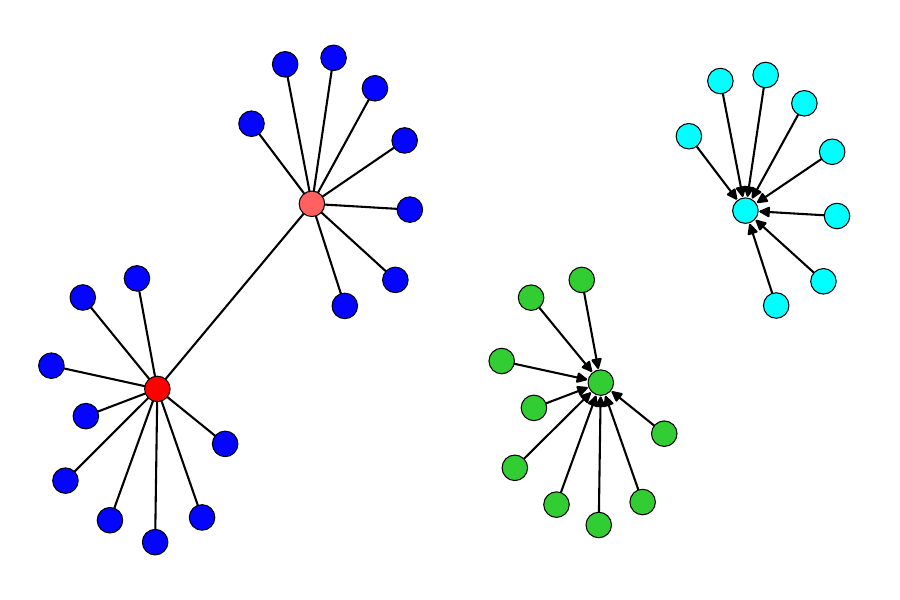}
 \caption{Morse clustering can separate nearby local maxima. Left: Toy graph with node weights shown by colour from low (dark blue) to high (red, critical vertex). Right: Output of the Morse flow and clustering algorithm (two clusters, shown by colour). The closest (highest weight) edge at each red vertex is a `downhill' edge to a blue vertex. Hence both vertices, and the edge between them, are critical.}
 \label{fig:local_maxima}
\end{figure}

Recall that an edge is critical if it does not participate in the flow, and a vertex is critical if it is a fixed point, $\Phi(v)=v$, that is, a `sink' of the flow.  Formally, we define
\begin{align*}
	V_\text{crit} = \{v \in V \mid \Phi(v) = v \} \quad \text{and} \quad
     E_\text{crit} = \{(v,w) \in E \mid \Phi(v) \neq w \}.
\end{align*}

We will see that, after removing the critical edges, what remains is a partition of the graph into a disjoint union of directed trees rooted at critical vertices (edge directions given by the flow). The cluster associated to a critical vertex $v$ is 
\begin{equation}\label{eqn:Tv}
	T_v =\{ w \in V \mid \Phi^N(w)=v \text{ for some } N \ge 0\}.
\end{equation}
Here $\Phi^N$ is the composition of $\Phi$ with itself $N$ times (and $\Phi^0$ is the identity map), so that $\Phi^N(w)$ is the vertex at which we arrive from $w$ after following the flow $N$ steps (across $N$ edges). In the dynamical system terminology, we can describe each $T_v$ as the `basin of attraction' of $v$. 

Let us write $\mathcal{T}_v$ for the subgraph with vertex set $T_v$ and edge set all non-critical edges between vertices in $T_v$. Recall that the \emph{depth} of a rooted tree is the maximal distance to its root. 


\begin{theorem}\label{thm:morseflow}
Let $\mathcal{G}=(V,E)$ be a finite graph with edge and vertex preorders $\preceq_E$ and $\preceq_V$, and associated Morse flow $\Phi \colon V \to V$. Then:
\begin{itemize}
\item[(i)] The Morse flow stabilises, that is, there is $N \ge 0$ such that $\Phi^N = \Phi^{N+1}$;
\item[(ii)] $\{T_v \mid v \in V_\text{crit}\}$ is a partition of $V$;
\item[(iii)] $\mathcal{T}_v$ is a directed (edge directions given by the flow) rooted tree with root $v$;
\item[(iv)] Within $T_v$, the vertex $v$ is the only critical vertex, and it is maximal with respect to the vertex preorder;
\item[(v)] $\max \{ \textup{depth}(T_v) \mid v \in V_\text{crit}\} = \min\,\{N\ge 0 \mid \Phi^N = \Phi^{N+1}\}$;
\item[(vi)] The graph $(V,E\setminus E_\text{crit})$ equals the disjoint union of the graphs $\mathcal{T}_v$ for $v \in V_\text{crit}$. 
\end{itemize}
\end{theorem}
\begin{proof}
(i) Let $v \in V$. By the definition of the Morse flow, either $\Phi(v)=v$ (a critical vertex), or $\Phi(v)=w$ and $v \prec_V w$ (which implies $v \neq w$ by reflexivity). Therefore, the sequence $v=v_0, v_1, v_2, \ldots$ where $v_i = \Phi^i(v)$, must contain a critical vertex before the first repetition: otherwise, we would have $v_i \prec_V v_{i+1} \prec_V \ldots \prec_V v_{k}=v_i$ and thus $v_i \prec_V v_{i}$ by transitivity, a contradiction. Since the graph is finite, say $|V|=n$, there will be repetition in any subset of $n+1$ vertices. Consequently, there is a critical vertex $v_k=\Phi^k(v)$ in the sequence above and, in fact, $k\le n$. All in all, the flow stabilises after at most $n=|V|$ steps. (The case $N=0$ can occur if all vertices are critical.)

(ii) Let $v \in V$. By the argument in (i), the sequence $v_i=\Phi^i(v)$ ($i\ge 0$) stabilises, that is, there is $k \ge 0$ such that $\{v_i \mid 0\le i \le k\}$ are distinct, non-critical, and $v_j=v_k$ critical for all $j \ge k$. In particular, $v \in T_{v_k}$, by Eq.~\eqref{eqn:Tv}. This shows that every vertex belong to a set $T_v$ for $v \in V_\text{crit}$, and that these sets must be disjoint. 

(iii) Since all edges in $\mathcal{T}_v$ are non-critical, we have $v \prec_V w$ across each edge and thus a cycle would imply $u \prec_V u$ for some vertex $u$, a contradiction. All edges are directed and point towards the root $v$, by the discussion above.

(iv) For each critical vertex $w$, we have $\Phi(w)=w$ hence $w \in T_w$. Since, by (ii), they form  a partition of the vertex set, $v$ is the only critical vertex in $T_v$. Every (directed) edge $(u,w)$ in $\mathcal{T}_v$ is not critical, hence admissible, so that $u \prec_V w$. As $v$ is the root of the tree $\mathcal{T}_v$, it must then be maximal with respect to $\preceq_V$. 

(v) It suffices to show that, for any $v\in V_\text{crit}$, and any $N \ge 0$ such that $\Phi^N=\Phi^{N+1}$, we have $\textup{depth}(\mathcal{T}_v) \le N$. Let $k=\textup{depth}(\mathcal{T}_v)$. Then there is $w \in T_v$ such that $w_i=\Phi^i(w)$, $i \ge 0$, stabilises after exactly $k$ steps. In particular, $w_0,\ldots,w_{k-1}$ are all distinct and hence $k \le N$. 

(vi) Let $\mathcal{G}_1=(V,E\setminus E_\text{crit})$ and $\mathcal{G}_2$ the disjoint union of the graphs $\mathcal{T}_v$ for $v\in V_\text{crit}$. Since $\{T_v \mid v \in V_\text{crit}\}$ is a partition of $V$, both $\mathcal{G}_1$ and $\mathcal{G}_2$ have the same vertex set. We show they also have the same edge set and hence they are equal. The edges in $\mathcal{G}_2$ are non-critical thus a subset of $E\setminus E_\text{crit}$. Conversely, given a non-critical edge $(v,w)$ in $\mathcal{G}_1$, we have $w=\Phi(v)$ and the sequence $v,\Phi(v)=w,\Phi^2(v),\ldots$ shows that $v$ and $w$ belong to the same tree critical tree, and thus this tree contains the edge $(v,w)$.
\end{proof}
\subsection{Morse clustering algorithm}\label{sec:MorseClusteringAlgorithm}
The \emph{Morse partition} of a graph $\mathcal{G}$ with a choice of vertex and edge preorders $\preceq_V$ and $\preceq_E$ is the partition of the vertex set given by the connected components of the graph $\mathcal{G}_\text{Morse}=(V,E\setminus E_\text{crit})$. By Theorem \ref{thm:morseflow}, there is a cluster for each critical vertex, and, in fact, $\mathcal{G}_\text{Morse}$ is a disjoint union of directed rooted trees with roots at the critical vertices.

A complete algorithm that returns the Morse clustering of $(\mathcal{G},\preceq_V,\preceq_E)$ is given below (Algorithm \ref{alg:morseclustering}). Its time complexity is clearly linear on the number of vertices and edges. Alternatively, the Morse flow and clustering can be computed simultaneously one edge at a time, by keeping a list of critical edges and of the maximal edge at each vertex. Therefore, the time complexity of (any instance of) Morse clustering is $\mathcal{O}(m)$ where $m$ is the number of edges of the graph. 

\begin{algorithm2e}
\SetAlgoLined
\DontPrintSemicolon
\SetKwFunction{MorseFlow}{MorseFlow}\SetKwFunction{Diag}{diag}
\SetKwData{EmptyGraph}{EmptyGraph}
\KwIn{graph $\mathcal{G}=(V,E)$, edge preorder $\preceq_E$, vertex preorder $\preceq_V$}
\KwOut{partition $\mathcal{P}$ of $V$}
\BlankLine
$n \leftarrow |V|$\;
initialise $G_\texttt{Morse}$ to an empty graph on $n$ vertices\;
$\Phi \leftarrow \MorseFlow(G,\preceq_E, \preceq_V)$\;
\BlankLine
\ForEach{$v \in V$}
	{\If{$\Phi(v) \neq v$}
		{add edge $(v,\Phi(v))$ to $G_\texttt{Morse}$\;}}
\BlankLine
$\mathcal{P} \leftarrow $ connected components of $G_\texttt{Morse}$\;
\caption{Morse clustering algorithm.}
\label{alg:morseclustering}
\end{algorithm2e}

We finish with a useful result, needed later, to determine when two Morse partitions are equal. 

\begin{lemma}\label{lem:refinement}
Let $\Phi$ and $\Phi'$ be Morse flows on $X$ with associated Morse partitions $\mathcal{P}$ and $\mathcal{P}'$. If $x \sim_\mathcal{P} \Phi'(x)$ for all $x \in X$, then $\mathcal{P}'$ is a refinement of $\mathcal{P}$.
\end{lemma}
\begin{proof}
Write $\mathcal{P}=\{X_1,\ldots,X_n\}$ and $\mathcal{P}'=\{X'_1,\ldots,X'_{n'}\}$. Write $x_i$, respectively $x'_j$, for the critical vertex in $X_i$, respectively $X'_j$, for all $i, j$.
Choose $N \ge 1$ such that both $\Phi$ and $\Phi'$ stabilise, that is, $\Phi^N=\Phi^{N+1}$ and $(\Phi')^N=(\Phi')^{N+1}$. We need to show that, for each $j$ there is $i$ such that $X'_j \subseteq X_i$. 

Let $x \in X'_j$ and consider the flow paths
\begin{align*}
	p(x) &= \{x, \Phi(x), \ldots, \Phi^N(x)=x_i\} \ \text{ and }\ 
    p'(x) = \{x, \Phi'(x), \ldots, (\Phi')^N(x)=x'_j\}
\end{align*}
By definition of Morse partition, all points in $p(x)$ are in the same cluster of $\mathcal{P}$, namely $X_i$, and all points in $p'(x)$ in the same cluster of $\mathcal{P}'$, namely $X'_j$. By hypothesis, $(\Phi')^n(x) \sim_\mathcal{P} (\Phi')^{n+1}(x)$ for all $n \ge 0$, so $p'(x) \subseteq X_i$. In particular 
$x'_j \sim_\mathcal{P} x_i$.

Given any other $y \in X'_j$,  
\begin{align*}
	p(y) &= \{y, \Phi(y), \ldots, \Phi^N(y)=x_k\} \subseteq X_k \ \text{ and }\ 
    p'(y) = \{y, \Phi'(y), \ldots, (\Phi')^N(y)=x'_j\} \subseteq X'_j,
\end{align*}
for a possibly different cluster $X_k$. Again, by hypothesis, we have $p'(y) \subseteq X_k$ and, in particular, $x'_j \sim_\mathcal{P} x_k$. Then $x'_j \in X_i \cap X_k \neq \emptyset$ and hence $i=k$, as distinct clusters are disjoint. Since $y$ was arbitrary, we conclude that $X'_j \subseteq X_i$.
\end{proof}

Note that two partitions are equal if and only if each is the refinement of the other, or if they have the same size (number of clusters) and one is the refinement of the other.  

\subsection{A possibility theorem for Monotonic Consistency} \label{sec:ThreeInstances}\label{sec:Possibility}
Morse Clustering depends on a choice of edge and vertex preorders on a given graph. 
Different choices of edge and vertex preorders result in different instances of Morse Clustering. We now show an instance of Morse Clustering that satisfies Scale-Invariance and Richness (Theorem \ref{thm:SiRMorse}) as well as Monotonic Consistency (Theorem \ref{thm:SiRMorseMonotonic}), proving that these three axioms are mutually compatible. 

\begin{remark}
One can in fact define two further instances of Morse Clustering, that we call \texttt{$k$-Morse} and \texttt{$\delta$-Morse}, which satisfy the other two pairs of Kleinberg's axioms, namely Consistency and Scale-Invariance and Consistency and Richness, resp. \texttt{$k$-Morse}  and \texttt{$\delta$-Morse}. Furthermore, they both satisfy Monotonic Consistency (see Table \ref{table:axioms} and \ref{appendix:instances}). 
\end{remark}

Let $(X,d)$ be a set with a distance function, and consider the complete graph with vertex set $X$. Let us fix, once and for all, a labelling $X=\{x_1,x_2,\ldots,x_n\}$, which we will use to create the vertex preorders (see the remarks at the end of this section on labelling). We also assume that $X$ has at least three points.

We now define an instance of Morse Clustering that we call $\texttt{SiR-Morse}$ (\texttt{S}cale-\texttt{i}nvariant and \texttt{R}ich). It is given the the following choice of vertex and edge preorders.

\begin{itemize}
 \item $v_i \preceq_V v_j $ if $i \le j$
 \item $(v, w) \preceq_E (v,t)$ if $d(v,w) \ge d(v,t)$
\end{itemize}

Note that the vertex preorder is a total order, and the edge preorder is also locally total (at each vertex). The corresponding Morse flow chooses, at each vertex $v$, the edge with smallest distance, if it is unique and admissible. On the other hand, if more than one edge at $v$ achieves the smallest distance, or if such edge is not admissible, then $v$ is critical, that is, the Morse flow fixes $v$, $\Phi(v)=v$. 

\begin{theorem}\label{thm:SiRMorse}
\texttt{SiR-Morse} is Scale-Invariant and Rich.
\end{theorem}
\begin{proof}
({\textbf{Scale-invariance}})
Scale-Invariance does not affect the vertex or edge preorders, since $\preceq_V$ is independent of $d$ and, for $\preceq_E$, $d(v,w) \le d(v,t)$ if and only if $\alpha\, d(v,w) \le \alpha\, d(v,t)$ for all $\alpha >0$. Hence the output of $\texttt{SiR-Morse}$ for $(X,d)$ and for $(X,\alpha d)$ are the same.

\smallskip
({\textbf{Richness}}) Consider $V=V_1\cup\ldots\cup V_k$ an arbitrary partition of $V$. Let $v_i$ be the maximal vertex in $V_i$ ($\preceq_V$ is a total order) and define a distance $d$ as follows
\[
d(v,w) = 
\begin{cases}
1, & \mbox{ if } v,w\in V_i\mbox{ for some $i$, and either }v = v_i\mbox{ or }w = v_i,\\
2, & \mbox{ otherwise},
\end{cases}
\]
for all $v \neq w$. 
If $v \in V_i$, the edge to $v_i$ is always admissible and the largest with respect to $\preceq_E$, so $\Phi(v)=v_i$ for the Morse flow, and we recover the partition $V_1\cup\ldots\cup V_k$. 
\end{proof}

\begin{table}
\centering
 \begin{tabular}{|c|c|c|c:c|}
  \hline
  \small & \footnotesize Scale-Invariance & \footnotesize Richness & \footnotesize Consistency & \footnotesize Monotonic-Consistency\\
  \hline
  \texttt{SiR-Morse} & \checkmark & \checkmark & \xmark & \checkmark\\
  \hline
  \texttt{$k$-Morse} & \checkmark & \xmark & \checkmark & \checkmark\\
  \hline
  \texttt{$\delta$-Morse} & \xmark & \checkmark & \checkmark & \checkmark\\
  \hline
 \end{tabular}
 \caption{Clustering axioms and three instances of Morse clustering.}
 \label{table:axioms}
\end{table}

Our main result is that this instance of Morse Clustering also satisfies Monotonic Consistency. 

\begin{theorem}\label{thm:SiRMorseMonotonic}
\texttt{SiR-Morse} satisfies Monotonic Consistency.
\end{theorem}
\begin{proof}
Let $d$ be a distance on $X$, $\mathcal{P}$ the output partition of \texttt{SiR-Morse} on $(X,d)$, and $d'$ a $\mathcal{P}$-monotonic transformation of $d$. We want to show that \texttt{SiR-Morse} produces the same partition on $(X,d')$. We will prove that, in fact, the associated Morse flows $\Phi$ and $\Phi'$ are identical.

Let $\eta$ be a monotonic transformation realising $d'$, that is, 
\[
\begin{array}{ll}
 d(u,v) = \eta(d'(u,v)) 		&\text{ if $u \sim_\mathcal{P} v$, and}\\
 d(u,v) = \eta^{-1}(d'(u,v))  &\text{ if $u \not\sim_\mathcal{P} v$.}
\end{array}
\]
Let $v \in X$ and consider first the case $w=\Phi(v)\neq v$. Then, by the definition of \texttt{SiR-Morse} preorders, 
\[
  d(v,w)<d(v,s) \text{ for all } s\neq v, w.
\]
To prove that $\Phi'(v) = w$, we need to show that $d'(v,w)<d'(v,s)$ for all $s\neq v, w$. We have two subcases. 
\begin{enumerate}
 \item If $s\sim_{\mathcal{P}} v$, we have $d'(v,w)=\eta^{-1}(d(v,w))$ and $d'(v,s)=\eta^{-1}(d(v,s))$, so
  \[
   d(v,w)<d(v,s) \text{ implies } d'(v,w)<d'(v,s),
  \]
 as $\eta^{-1}$ is increasing (Lemma \ref{lem:expansive-properties}). 
 \item If $s\not\sim_{\mathcal{P}} v$, we have $d'(v,w)=\eta^{-1}(d(v,w))$ and $d'(v,s)=\eta(d(v,s))$, so 
  \[
   d(v,w)<d(v,s) \text{ implies } d'(v,w)\le d(v,w)<d(v,s)\le d'(v,s),
  \]
as $\eta^{-1}(x)\le x\le \eta(x)$ for all $x$ (Lemma \ref{lem:expansive-properties}). 
\end{enumerate}
In conclusion, we have $d'(v,w)<d'(v,s)$ for all $s\neq v,w$ so $\Phi'(v) = w$.

The remaining case is $\Phi(v)=v$. Suppose, by contradiction, that $w=\Phi'(v)\neq v$. This implies $v\prec_V w$ and $d'(v,w)<d'(v,s)$ for all $s\neq v,w$. Note that, since $v$ is critical and therefore maximal within its cluster, we have $v \not\sim_\mathcal{P} w$. On the other hand, $\Phi(v) = v$ means that either the unique maximal edge is not admissible, or it is admissible but the maximum is not unique. 

First we show that $d(v,w)$ is also a minimal distance at $v$ (possibly not unique). Suppose, by contradiction, $d(v,s) < d(v,w)$ for some $s \neq v, w$. There are two subcases.
\begin{enumerate}
 \item If $s\sim_{\mathcal{P}} v$, then we have $d'(v,s)=\eta^{-1}(d(v,s))$ and $d'(v,w)=\eta(d(v,w))$, so
  \[
   d(v,s)<d(v,w) \text{ implies } d'(v,s)\le d(v,s)<d(v,w)\le d'(v,w),
  \]
 as $\eta^{-1}(x)\le x\le \eta(x)$ (Lemma \ref{lem:expansive-properties}).
 \item If $s\not\sim_{\mathcal{P}} v$, then we have $d'(v,s)=\eta(d(v,s))$ and $d'(v,w)=\eta(d(v,w))$, so 
  \[
   d(v,s)<d(v,w) \text{ implies } d'(v,s)<d'(v,w),
  \]
as $\eta$ is increasing (Lemma \ref{lem:expansive-properties}). 
\end{enumerate}
In either case, we have $d'(v,s) < d'(v,w)$, a contradiction to the minimality of $d'(v,w)$. 

Since $d(v,w)$ is a minimal distance and $v \prec_V w$, but $\Phi(v)=v\neq w$, the minimal distance (maximal edge) cannot be unique. Let $d(v,s)=d(v,w)$ for some $s \neq v, w$. We have, again, two subcases. 
\begin{enumerate}
 \item If $s\sim_{\mathcal{P}} v$, then we have $d'(v,s)=\eta^{-1}(d(v,s))$ and $d'(v,w)=\eta(d(v,w))$, so
  \[
   d(v,s)=d(v,w) \text{ implies } d'(v,s)\le d(v,s)=d(v,w)\le d'(v,w),
  \]
 as $\eta^{-1}(x)\le x\le \eta(x)$ (Lemma \ref{lem:expansive-properties}).
 \item If $s\not\sim_{\mathcal{P}} v$, then we have $d'(v,s)=\eta(d(v,s))$ and $d'(v,w)=\eta(d(v,w))$, so 
  \[
   d(v,s)=d(v,w) \text{ implies } d'(v,s)=d'(v,w),
  \]
  as $\eta$ is injective (Lemma \ref{lem:expansive-properties}).
\end{enumerate}
This implies that $d'(v,s)\le d'(v,w)$, so $d'(v,w)$ cannot be the unique minimal distance for $d'$ at $v$, a contradiction.
\end{proof}

\begin{corollary}\label{cor:Possibility}
Scale Invariance, Richness and Monotonic Consistency are mutually compatible clustering axioms. 
\end{corollary}



\section{Axiomatic Approach to Graph Clustering }\label{sec:Impossibility}
In this section, we consider the axiomatic approach in the context of graph clustering, that is, of distances supported on a given graph $G$. Mathematically, we allow the distance function to take the value $0$ (Definition \ref{def:pseudo-distance}). Conceptually, there are two different approaches depending on whether $0$ is considered a numerical value (minimum distance) or indicating that the distance is `not defined'. 
The first approach is essentially equivalent to that of \cite{laarhoven:2014}, where a possibility theorem for Kleinberg's axioms is shown. The second approach, on the other hand, is closer to the usual interpretation of graph clustering, or partitioning, in network and computer science \cite{newman2018networks, schaeffer:2007}. In this case, we prove an impossibility result for Consistency (Section \ref{sec:ImpossibilityGraph}), even when Richness is replaced by the more natural Connected-Richness axiom, and a possibility result for Monotonic Consistency (Section \ref{sec:PossibilityGraph}). First, we discuss the two approaches. 

\subsection{Two approaches to graph clustering}\label{sec:Duality}
If we allow a distance function to take the value 0 between pairs of distinct points, we obtain what we call a pseudo-distance. 
\begin{definition}\label{def:pseudo-distance}
A \emph{pseudo-distance on a set} $X$ is a function $d:X\times X \to \R$ such that $d(v,w)=d(w,v) \ge 0$ and $d(v,v)=0$ for all $v, w \in X$ (that is, we allow $d(u,v)=0$ for $u \neq v$).
\end{definition}
We can represent a pseudo-distance on $X$ as a graph with vertex set $X$ in the usual way: an edge between $i$ and $j$ if $d(i,j)>0$, and no such edge if $d(i,j)=0$. Pseudo-distances occur naturally in network clustering or community detection \cite{fortunato2010community} to represent absent edges, as well as in distance measures \cite{xu2015comprehensive} that allow 0 values, such as the Pearson correlation distance or the cosine distance for unnormalised vectors.

A direct generalisation of Kleinberg's definition of clustering algorithm, Eq.~\eqref{eqn:clus-algo-Kleinberg}, is
\begin{equation}\label{eq:CA1}
	F:\{ d \text{ pseudo-distance on } X\}\to\{{\mathcal P}\mbox{ partition of } X\}.
\end{equation}
Kleinberg's original axioms make sense in this setting, however a possibility result now holds: the function that returns the connected components of the graph representation of $d$ (as above), is clearly scale-invariant, rich and consistent (cf.~\cite{laarhoven:2014}). 

The main issue with Eq.~\eqref{eq:CA1} is that the graph becomes irrelevant: although Scale-Invariance does not change the underlying graph, the Consistency axiom can create and eliminate edges, by setting their values to zero, or not zero. Hence this approach focuses on the set $X$ rather than on a fixed graph $G$. 


\begin{remark}
In \cite{laarhoven:2014} the authors define a graph as a pair $V$ vertex set and $E \colon V \times V \to \mathbb{R}^{\ge0}$, with $0$ effectively signifying the lack of an edge. Their Consistence axiom (which they call Consistency Improvement), allows $E'(i,j) \ge E(i,j)$ whenever $i \sim_C j$ and  $E'(i,j) \le E(i,j)$ whenever $i \not\sim_C j$, for a clustering $C$ of the same vertex set $V$. In particular, we are allowed to create or eliminate edges by setting $E'(i,j) > E(i,j)=0$, respectively  $E'(i,j) = 0 < E(i,j)$. 
\end{remark}



Instead, we suggest a more natural approach when the focus is on the graph $G=(V,E)$: we allow arbitrary positive distances on edges while keeping $d(u,v)=0$ whenever $(u,v) \not\in E$. In clustering problems, we are normally interested in minimising the edge cut \cite{schaeffer:2007}, and hence the absence of an edge is significant. In fact, the underlying hypothesis in graph clustering is that the structure of the graph, or network, carries information. For this reason, we fix a graph $G$ and restrict to distances supported on (the edges of) $G$, and define distances, and clustering algorithms, accordingly. 

\begin{definition}\label{def:pseudo-distance-revised}
A \emph{pseudo-distance on a graph} $G=(V,E)$ is a pseudo-distance $d$ on the vertex set $V$ that is supported on the edge set, that is, $d(v,w)\neq0$ if and only if $(v,w)\in E$. (Equivalently, a positive weight function on undirected edges.) 
\end{definition}
Note that, for this definition to make sense, $G$ must be loopless and undirected (we will assume this from now on). Given a graph $G=(V,E)$, we define a \emph{graph clustering algorithm} as any function 
\begin{equation}\label{eqn:pseudoclustering}
	F:\{ d \text{ pseudo-distance on } G \}\to\{{\mathcal P}\mbox{ partition of } V\}.
\end{equation}
Clearly, a distance on a set $X$ is the same as a pseudo-distance on the complete graph with vertex set $V=X$. Hence this so-called \emph{sparse} setting generalises Kleinberg's setting from a complete to an arbitrary (but fixed) graph on $X$.

\subsection{An impossibility theorem for graph clustering}\label{sec:ImpossibilityGraph}
Kleinberg's axioms can be stated in the graph clustering setting above (Eq.~\eqref{eqn:pseudoclustering}), as follows. 

\begin{itemize}
\item \textbf{Scale-invariance}: For any pseudo-distance $d$ on $G$ and $\alpha>0$, we have $F(d) = F(\alpha\cdot d)$;
 \item \textbf{Richness}: Given a partition ${\mathcal P}$, there exists a pseudo-distance $d$ on $G$ such that $F(d) = {\mathcal P}$;
 \item \textbf{Consistency}: Given pseudo-distances $d$ and $d'$ on $G$ with ${\mathcal P} = F(d)$, if $d'$ is a $\mathcal{P}$-transformation of $d$, that is, 
 \begin{eqnarray}
 \begin{cases}
 d'(v,w)\leq d(v,w)\ \mbox{ if $v \sim_\mathcal{P} w$, and}\\
 d'(v,w)\geq d(v,w)\ \mbox{ if $v \not\sim_\mathcal{P} w$,}
 \end{cases}
 \end{eqnarray}
then $F(d') = F(d)$.
\end{itemize}
(If $G$ is a complete graph these axioms coincide with Kleinberg's for the set $X=V$.) 

In the sparse setting it seems natural to restrict to \emph{connected partitions}, that is, partitions where each cluster is a connected subgraph of $G$. Otherwise, we would be grouping together objects which are unknown to be similar or not, in apparent contradiction with the very principle of clustering. Therefore, we define a weaker Richness axiom:
\begin{itemize}
 \item \textbf{Connected-Richness}: Given a connected partition ${\mathcal P}$, there exists a pseudo-distance $d$ on $G$ such that $F(d) = {\mathcal P}$.
\end{itemize}
Similarly, we will only consider connected graphs from now on (it seems sensible to assume $F(G)=F(G_1) \cup F(G_2)$ whenever $G$ is the disjoint union of graphs $G_1$ and $G_2$).

Connected-Richness is clearly equivalent to Richness in the complete case. In the sparse case, however, many graph clustering algorithms, such as Single Linkage, or Morse Clustering (Algorithms \ref{alg:morseflow} and \ref{alg:morseclustering}), always produce a connected partition (which seems very sensible in any case). Since clustering algorithms cannot create new edges, such algorithms cannot satisfy Richness in its general form. Since Richness implies Connected-Richness, our impossibility result also holds for Scale-Invariance, Consistency and Richness. 


\begin{theorem}[An Impossibility Theorem for Graph Clustering] \label{thm:ImpossibilityGraphs}
Let $G$ be a connected graph with at least three vertices, and $F$ a graph clustering algorithm on $G$. Then $F$ cannot satisfy Scale-Invariance, Consistency and Connected-Richness.
\end{theorem}
Before proving this theorem, we introduce some notation. Given a pseudo-distance $d$ on $G=(V,E)$ and a partition $\mathcal{P}$ of $V$, let $g(\mathcal{P},d)=(x,y)$ and $h(\mathcal{P},d)=(p,q)$ where
\begin{align*}
 x &= \max \left\{d(u,v)\mid (u,v) \in E, u \sim_\mathcal{P} v\right\},
   & p = \min \left\{d(u,v)\mid (u,v) \in E, u \sim_\mathcal{P} v\right\},\\
 y &= \min \left\{d(u,v)\mid (u,v) \in E, u \not\sim_\mathcal{P} v\right\},
  & q = \max \left\{d(u,v)\mid (u,v) \in E, u \not\sim_\mathcal{P} v\right\},
\end{align*}
the maximal (minimal) intra (inter) cluster distances, and, if $\mathcal{P}$ is the trivial partition, we set $y=q=0$.  

We observe that, if $d$ and $d'$ are pseudo-distances on $G$ and $\mathcal{P}$ is a partition of $V$, the condition $h(\mathcal{P},d)=g(\mathcal{P},d')$ guarantees that $d'$ is a $\mathcal{P}$-transformation of $d$. 


\medskip
\begin{proof}
Note that, in any connected graph, we can always remove a vertex so that the remaining graph is connected. For example, if $T$ is a spanning tree of $G$, $v$ any vertex, and $s$ the vertex realising the maximal (shortest path) distance from $v$ in $T$, then the graph induced by $V \setminus \{s\}$ must still be connected. Since $|V| \ge 3$, we can repeat the argument on $V \setminus \{s\}$ and find $t \neq s$ such that $\mathcal{P} = \{ \{s\}, X\setminus \{s\} \}$ and $\mathcal{P}' = \{\{s\},\{t\}, X\setminus\{s,t\}\}$ are connected partitions. 

Since $F$ satisfies Connected-Richness, there exist pseudo-distances $d$ and $d'$ on $G$ such that $F(d) = \mathcal{P}$ and $F(d')=\mathcal{P}'$. Let $h(\mathcal{P},d)=(p,q)$ and $h(\mathcal{P}',d')=(p',q')$. Since $F$ satisfies Consistency, we can assume $p<q$ and $p'<q'$. Also, note that $p$, $q$ and $q'$ cannot be zero. 

Let $d^*$ be the pseudo-distance on $G$ defined by $d^*(s,v) = q$ if $v \neq s$, $d^*(t,v) = p$ if $v \neq s, t$, and $d^*(u,v)=(pp')/q'$ if $u, v  \neq s, t$. Then $g(\mathcal{P}, d^*) = (p,q)$, since the only inter-cluster distance value is $q$, and the only intra-cluster distance values are $p$ and $p(p'/q')<p$. Therefore, $g(\mathcal{P}, d^*) = h(\mathcal{P},d)$, hence $d^*$ is a $\mathcal{P}$-transformation of $d$, by the observation before the proof, and, consequently, $F(d^*) = F(d)$, by Consistency. 

On the other hand, $g(\mathcal{P}', \alpha d^*) =\alpha g(\mathcal{P}', d^*)$ for any $\alpha$ positive constant. If we choose $\alpha = q'/p$ then we have  $g(\mathcal{P}', \alpha d^*) = \alpha((pp')/q',p) = (p',q') = h(\mathcal{P}',d')$ so, by the same argument as above, $\alpha d^*$ is a $\mathcal{P}'$-transformation of $d'$ and thus $F(\alpha d^*) = F(d') = \mathcal{P}'$, by Consistency. Since $F$ satisfies Scale-Invariance, this implies $F(\alpha d^\ast) = F(d^*) = F(d) = \mathcal{P}$ and, therefore,
$\mathcal{P} = \mathcal{P}'$, clearly a contradiction.
\end{proof}

\subsection{Monotonic Consistency for graph clustering}\label{sec:PossibilityGraph}
Next we consider Monotonic Consistency and \texttt{Morse} Clustering in the sparse setting. We can extend Monotonic-Consistency to connected graphs by considering monotonic transformations (Definition \ref{def:monotonic-transformation}) of pseudo-distances on a given graph.
\begin{itemize}
 \item \textbf{Monotonic-Consistency}: Given pseudo-distances $d$ and $d'$ on $G$ with ${\mathcal P} = F(d)$, if $d'$ is a $\mathcal{P}$-monotonic transformation of $d$,
then $F(d') = F(d)$.
\end{itemize}

The input of the Morse Clustering algorithm (Algorithm \ref{alg:morseclustering}) is an arbitrary graph, and the output flow always induces a connected partition ({Theorem \ref{thm:morseflow}}). Therefore, we can consider Morse Clustering, and hence any of its instances, as graph clustering algorithms. 

The three instances of Morse Clustering discussed in Section \ref{sec:ThreeInstances} (and \ref{appendix:instances}) satisfy the analogous axioms as in the complete case except that we need to allow the vertex labelling (arbitrary but prefixed in the complete case) to be part of the algorithm to satisfy Connected-Richness. This is a necessary condition: once a vertex labelling (or preorder) is fixed, only `uphill' edges are admissible, preventing certain configurations to occur (for example, $u$ and $v$ cannot be in the same cluster if all paths from $u$ to $v$ contain a vertex lower than both). This is not an intrinsic limitation of Morse Clustering but reflects the fact that it is fundamentally a vertex-weighted clustering algorithm, that is, both distance and vertex preorder are part of the input data. 

We can either allow the (so far arbitrary and prefixed) vertex labelling to be part of the algorithm, or to restrict to partitions compatible with such a choice of vertex labelling. Formally, given a vertex preorder $\preceq_V$ on $V$, we say that a partition $\mathcal{P}=\{V_1,\ldots,V_k\}$ of $V$ is \emph{compatible with $\preceq_V$} if there is a rooted spanning tree $T_i$ of (the subgraph induced by) $V_i$ rooted at a vertex $v_i$ such that every directed edge in $T_i$ (edges directed towards the root) is admissible with respect to $\preceq_V$. Note that $v_i$ is necessarily the maximal vertex in $T_i$ with respect to the preorder, and that $\mathcal{P}$ is necessarily a connected partition. 

\begin{remark}
One can show that $\mathcal{P}$ is compatible with $\preceq_V$ if and only if for every $u \sim_\mathcal{P} v$ there exists a path from $u$ to $v$ such that no vertex in the path is strictly less than both $u$ and $v$. 
\end{remark}

Clearly, for every partition there is a choice of compatible preorder $\preceq_V$. This is also true for the \texttt{SiR} and \texttt{$\delta$-Morse} vertex preorders: given a partition, there is a choice of labelling $V=\{v_1,\ldots,v_n\}$ such that the preorder is compatible with the partition (Section \ref{sec:ThreeInstances}, \ref{appendix:instances}). 

Formally, we define Morse-Richness for a Morse clustering algorithm $F$ on a graph $G=(V,E)$ with a choice of vertex preorder $\preceq_V$ as follows.

\begin{itemize}
 \item \textbf{Morse-Richness}: Given a partition ${\mathcal P}$ of $V$ compatible with $\preceq_V$, there exists a pseudo-distance $d$ on $G$ and a vertex preorder such that $F(d) = {\mathcal P}$.
\end{itemize}

(Morse-Richness is thus equivalent to Connected-Richness if we accept the vertex labelling as an input of the algorithm.)

Now we can show that the three instances of Morse Clustering satisfy the analogous axioms as in Section \ref{sec:MorseClustering} (see Table \ref{table:axioms}), including a possibility theorem for Monotonic-Consistency and \texttt{SiR-Morse}.

\begin{theorem}\label{thm:PossibilityGraph}
Let $G=(V,E)$ be a graph, and consider \texttt{SiR-Morse}, \texttt{$k$-Morse} and \texttt{$\delta$-Morse} as graph clustering algorithms on $G$, for some fixed labelling $V=\{v_1,\ldots,v_n\}$. Then:
\begin{enumerate}
\item[(i)] \texttt{SiR-Morse} satisfies Scale-Invariance, Morse-Richness and Monotonic Consistency.
\item[(ii)] \texttt{$k$-Morse} satisfies Scale-Invariance and Consistency. 
\item[(iii)] \texttt{$\delta$-Morse} satisfies Morse-Richness and Consistency.
\end{enumerate}
\end{theorem}
\begin{proof}
\fbox{i} The proofs of Scale Invariance and Monotonic Consistency are identical (they do not use the fact that $G$ is a complete graph) as those in Theorem \ref{thm:SiRMorse}. For Morse-Richness, consider $V =V_1 \cup \ldots \cup V_k$ an arbitrary connected partition of $V$. For each $V_i$, choose a spanning tree $T_i$ and a root $v_i$ such that each edge in $T_i$ is admissible. 

Define a pseudo-distance $d$ on $G$ as follows. If $(s,t)$ is an edge on $T_i$, then $d(s,t)$ is the maximum of the distance from $s$ to $v_i$ in $T_i$ and the distance from $t$ to $v_i$ in $T_i$ (by distance in a tree we simply mean the `hop' distance). If $(s,t)$ is an edge not in any spanning tree, then $d(s,t)=|V|$.

With this choice, $v_i$ is critical and, if $v \in V_i$, then the maximal edge at $v$ is the one connecting it to a vertex in $T_i$ closer to $v_i$, and it is admissible. All in all, the associated tree $T_{v_i} = T_i$ and the Morse flow recovers the original partition. 

\smallskip
\fbox{ii} The proof of Scale Invariance is identical to that in Theorem \ref{thm:kMorse}. For Consistency, let $d$ be a pseudo-distance on $G$, $\mathcal{P}$ the partition given by $\texttt{$k$-Morse}$, and $d'$ a $\mathcal{P}$-transformation of $d$, that is,
\begin{eqnarray*}
\begin{cases}
 d(v,w)\ge d'(v,w), &\ \mbox{ if } v \sim_\mathcal{P} w,\\ 
 d(v,w)\le d'(v,w), &\ \mbox{ otherwise}. 
\end{cases}
\end{eqnarray*}
Let $\Phi$ respectively $\Phi'$ be the Morse flow corresponding to $d$ respectively $d'$. As in the proof of Theorem \ref{thm:kMorse}, for all $i>n-k$ we have that $\Phi(v_i)=v_i=\Phi'(v_i)$, critical. 

Suppose now $\Phi(v_i) = v_i$ for some $i\le n-k$. Let $J = \{v_j\mid (v_i,v_j)\in E, v_i\prec_V v_j\}$, the admissible edges from $v_i$. By the definition of the edge preorder, if there are admissible edges ($J\neq \emptyset$) then the maximal admissible edge exists and it is unique. Since $v_i$ is critical, we must have $J=\emptyset$. Since there are no admissible edges at $v_i$, we also have $\Phi'(v_i) = v_i$. All in all, $\Phi$ and $\Phi'$ have the same number of critical points and therefore $\mathcal{P}$ and $\mathcal{P}'$ have the same number of clusters (possibly more than $k$). The rest of the proof goes as in the proof of Theorem \ref{thm:kMorseAxioms}.

\smallskip
\fbox{iii} The proof of Consistency is identical to that in Theorem \ref{thm:deltaMorse}. For Morse-Richness, consider $V =V_1 \cup \ldots \cup V_k$ an arbitrary connected partition of $V$, and choose a spanning tree $T_i$ and a root $v_i$ such that each edge in $T_i$ is admissible. 

Define a pseudo-distance $d$ on $G$ as follows. If $(s,t)$ is an edge in some $T_i$, then $d(s,t)=\delta/2$, and if $(s,t)$ is not an edge in any $T_i$ then $d(s,t)=\delta$. By the definition of edge preorder, $v_i$ is critical and the maximal edge at $v \in V_i\setminus \{v_i\}$ is the only edge in $T_i$ connecting $v$ to a vertex closer to $v_i$ in $T_i$. All in all, the tree associated to $v_i$ by the Morse flow is $T_i$ and hence we recover the original partition. 
\end{proof}

\section{Conclusions}
In this article, we presented a novel weakening of Kleinberg's Consistency axiom, called Monotonic Consistency, which avoids its well-known problematic behaviour and is compatible with the other two axioms without replacing Kleinberg’s original definition of clustering algorithm. As far as we know, this is the only possibility theorem which only modifies the Consistency axiom while keeping Kleinberg's original set-up. Monotonic Consistency explicitly avoids Kleinberg's Consistency problematic behaviour by restricting the rate of expansion and contraction of the inter- and intra-cluster distances. 


In the process of proving our possibility theorem, we introduced a family of clustering algorithms inspired by Morse Theory in Differential Topology that uncovers the underlying flow structure in the natural graph representation of a data set with a similarity function. Three different instances of these algorithms satisfy each pair of Kleinberg's axioms, and one instance, \texttt{SiR Morse}, shows the possibility result for Monotonic Consistency. Although Morse clustering is, in full generality, a vertex-weighted algorithm \cite{strazzeri:2018}, here it is used in an `agnostic’ way for axiomatic purposes only: the vertex weights are an arbitrary, but fixed, labelling of the vertices. 

Lastly, we generalised Kleinberg’s axiomatic setting to graph clustering, including the impossibility result, and a possibility result for Monotonic Consistency and Morse clustering. These results generalise the previous ones in our paper from distances to pseudo-distances, where we allow zero values between distinct points. This is a more natural setting for graph clustering and community detection in networks, where a 0 weight simply means the absence of an edge, that is, the distance is `not defined', rather than the actual 0 numerical value. 

Although introduced in our work for axiomatic purposes only, it would be interesting to study Morse Clustering on its own, as a family of clustering algorithms for vertex-annotated data \cite{strazzeri:2018}. Moreover, we would like to know whether there are other clustering algorithms that satisfy our possibility result (the usual candidate, single-linkage clustering, does not) and whether there are any uniqueness results that characterise Morse clustering. 



\acks{We thank Francisco Belch\'i Guillam\'on for valuable discussions about Monotonic Consistency and metrics. FS was supported by a PhD studentship by Mathematical Sciences and the Institute for Life Sciences at the University of Southampton. RSG was partially supported by The Alan Turing Institute under the EPSRC grant EP/N510129/.}


\bibliography{Bibliography}

\appendix
\section{Two further instances of Morse Clustering}\label{appendix:instances}
In this appendix, we define and study the two further instances of Morse clustering mentioned in the Main Text (Section \ref{sec:ThreeInstances}), namely \texttt{$k$-Morse} and \texttt{$\delta$-Morse}. They illustrate the versatility of Morse clustering and show that, for suitable choices of vertex and edge preorders, Morse clustering can satisfy each pair of Kleinberg's axioms, in analogy to the three instances of Single-Linkage clustering with the same property in \cite{kleinberg:2003}. We keep the notation and terminology from Section \ref{sec:ThreeInstances}.

Let $k \ge 1$ be an integer. First, we present a Morse algorithm that guarantees a partition with $k$ clusters (Theorem \ref{thm:kMorse}), and thus it cannot be rich. However, it satisfies Consistency and Scale Invariance (Theorem \ref{thm:kMorseAxioms}). We call it \texttt{$k$-Morse}, and it corresponds to the following choice of preorder. 

\begin{itemize}
 \item $v_i \preceq_V v_j$ if $i=j$ or $i + k < j$
 \item $(v, w) \preceq_E (v,t) $ if
 \begin{description}
  \item $w \preceq_V v \preceq_V t$, or
  \item $d(v,w) > d(v,t)$ and $v \preceq_V t$, or
  \item $d(v,w) = d(v,t)$ and $w \preceq_V t$.
 \end{description}
\end{itemize}

For this choice of vertex preorder, there are exactly $k$ critical vertices, $v_n$, $v_{n-1}$, $\ldots$, $v_{n-k+1}$, and hence $k$ clusters (see Theorem \ref{thm:kMorse} below). The edge preorder is defined such that admissible edges are always greater than non-admissible ones, and admissible ones are compared using distances, with the vertex preorder used as tie-breaking procedure. In particular, if there are admissible edges at $v$, the maximal admissible edge at $v$ exists and it is unique.  

\begin{theorem}\label{thm:kMorse}
\texttt{k-Morse} always produces a partition with $k$ clusters.
\end{theorem}
\begin{proof}
If $v_i \in X$ with $i>n-k$ then there are no vertices greater than $v_i$ with respect to $\preceq_V$ hence no admissible edges at $v$ and thus $\Phi(v_i)=v_i$ critical. On the other hand, $v_i$ with $i \le n-k$ cannot be critical, as there are admissible edges $(v_i,v_j) \in E_{v_i}$ for all $j > i+k$, so the maximum exists and it is unique. All in all, there are exactly $k$ critical vertices $v_n, v_{n-1}, \ldots, v_{n-k+1}$ and therefore exactly $k$ clusters. 
\end{proof}

\begin{theorem}\label{thm:kMorseAxioms}
\texttt{k-Morse} is Consistent and Scale-Invariant.
\end{theorem}
\begin{proof}
({\textbf{Scale-invariance}}) A distance transformation $d' = \alpha\cdot d$ for $\alpha>0$ does not affect the \texttt{$k$-Morse} vertex or edge preorder, hence we obtain the same partition.

\smallskip
({\textbf{Consistency}}) Let $d$ be a distance in $X$, $\mathcal{P}$ the partition given by $\texttt{$k$-Morse}$ on $(X,d)$, and $d'$ a $\mathcal{P}$-transformation of $d$, that is,
\begin{align}
 d(v,w)\ge d'(v,w), &\ \mbox{ if } v \sim_\mathcal{P} w, \label{eqn:1aa}\\
 d(v,w)\le d'(v,w), &\ \mbox{ otherwise}. \label{eqn:1bb}
\end{align}
Let $\Phi$ respectively $\Phi'$ be the Morse flow corresponding to $d$ respectively $d'$. The critical points depend on the vertex preorder alone, hence, as in the proof of Theorem \ref{thm:kMorse}, we have $\Phi(v_i)=v_i=\Phi'(v_i)$ for all $i > n-k$ and thus $\mathcal{P}$ and $\mathcal{P}'$ have the same number of clusters. Therefore, it suffices to show that $x \sim_\mathcal{P} \Phi'(x)$ for all $x\in X$, by Lemma \ref{lem:refinement}. 

Let $x \in X$. If $x$ is critical, $\Phi(x)=\Phi'(x)$ as they have the same critical points, so clearly $x \sim_\mathcal{P} \Phi(x)=\Phi'(x)$. If $x$ is not critical, let $w =\Phi(x)$ and $t=\Phi'(x)$. The maximality and the definition of $\preceq_E$ implies
$d(x,w) \le d(x,t) \text{ and } d'(x,t) \le d'(x,w)$.
Since $\Phi(x)=w$, they are in the same cluster, $x \sim_\mathcal{P} w$, and thus $d'(x,w) \le d(x,w)$, by Eq.~\eqref{eqn:1aa} above. All in all, 
\begin{gather}\label{eqn:longeq}
  d'(x,t)\leq d'(x,w)\leq d(x,w)\leq d(x,t).
\end{gather}
Now, if $d'(x,t) < d(x,t)$, they are necessarily in the same cluster, $x \sim_\mathcal{P} t$, by Eqs.~\eqref{eqn:1aa} and \eqref{eqn:1bb} above. The remaining case $d'(x,t)=d(x,t)$ implies equalities in Eq.~\eqref{eqn:longeq}, and, by the definition of the edge preorders and the maximality of $(x,w)$ with respect to $d$, we have $w=t$. In both cases, $x \sim_\mathcal{P} t = \Phi'(x)$. 
\end{proof}

Let $\delta >0$. The final instance of Morse clustering, called \texttt{$\delta$-Morse}, satisfies Consistency and Richness, and is given by the following choices of preorders.

\begin{itemize}
 \item $v_i \preceq_V v_j$ if $i \le j$
 \item $(v, w) \preceq_E (v,t)$ if
 \begin{description}
  \item $w = t$, or
  \item $d(v,t) < \min\{d(v,w), \delta\}$ and $v\preceq_V t$, or
  \item $d(v,w) = d(v,t) < \delta$ and $v\preceq_V w \preceq_V t$.
 \end{description}
\end{itemize}

With this preorder, only admissible edges with distance less than the threshold parameter $\delta$ are considered for the flow. Among those edges, we choose the one with minimal distance, using the vertex preorder to resolve ties. Note that, if there are admissible edges at distance less than $\delta$, the maximum admissible edge exists and it is unique. 

\begin{theorem}\label{thm:deltaMorse}
\texttt{$\delta-$Morse} satisfies Consistency and Richness.
\end{theorem}
\begin{proof}
(\textbf{Richness}) Consider an arbitrary partition $X=X_1\cup\ldots \cup X_k$ and define the distance function
\[
d(v,w)=
\begin{cases}
\frac{\delta}{2},\mbox{ if $v, w$ are in the same cluster, and}\\
\delta,\mbox{ otherwise,}
\end{cases}
\]
for $v \neq w$. 
Let $x_i$ be the largest vertex in $X_i$ with respect to $\preceq_V$ and $v \in X_i$ arbitrary. 
By the definition of $d$ and the edge preorder, we have that $(v,x_i)$ is the maximum admissible edge at $v$. Also, $x_i$ is critical: the maximum edge at $x_i$ is of the form $(x_i,w)$ for $w \in X_i$, hence not admissible or, if $|X_i|=1$, any edge in $E_{x_i}$ is maximal, hence unique (since $|X|\ge 3$). Therefore, $\delta-$Morse reproduces the partition $X_1\cup\ldots \cup X_k$ (in fact, each cluster is a directed star with root $x_i$).

\smallskip
(\textbf{Consistency}) Let $d$ be a distance in $X$, $\mathcal{P}$ the partition given by $\texttt{$\delta$-Morse}$ on $(X,d)$, and $d'$ a $\mathcal{P}$-transformation of $d$, that is,
\begin{align}
 d(v,w)\ge d'(v,w), &\mbox{ if } v \sim_\mathcal{P} w, \label{eqn:11}\\
 d(v,w)\le d'(v,w), &\mbox{ otherwise}. \label{eqn:12}
\end{align}
Let $\Phi$ respectively $\Phi'$ be the Morse flow corresponding to $d$ respectively $d'$. Let $s \in X$ arbitrary, $v = \Phi(s)$ and $w=\Phi'(s)$ with $v, w \neq s$. As in the proof of Theorem \ref{thm:kMorseAxioms}, we have 
$
  d'(s,w)\leq d'(s,v)\leq d(s,v)\leq d(s,w)
$.
Then either $d'(s,w) < d(s,w)$, and so $s \sim_\mathcal{P} w$ by Eq.~\eqref{eqn:11}, or $d'(s,w) = d(s,w)$, which implies, by the definition of edge preorder, $v = w$, and thus $s \sim_\mathcal{P} w = \Phi'(s)$ too. 
As $s$ was arbitrary, we conclude that $\mathcal{P}'$ is a refinement of $\mathcal{P}$, by Lemma \ref{lem:refinement}. To prove that they are equal, it suffices to show that they have the same critical points (i.e.~the same number of clusters), that is, $\Phi(v)=v$ if and only if $\Phi'(v)=v$.  

Suppose that $\Phi(v_i)=v_i$ and $\Phi'(v_i)=v_j$, $i \neq j$. Since the vertex preorder is strictly increasing along the flow, $v_i \prec_V v_j$, that is, $i < j$. By the definition of Morse clustering, $v_i \sim_\mathcal{P'} v_j$ hence $v_i \sim_\mathcal{P} v_j$, since $\mathcal{P'}$ is a refinement. However, this contradicts $v_i$ being maximal in its $\mathcal{P}$ cluster as $i < j$. 

Now suppose $\Phi'(v_i)=v_i$ and $\Phi(v_i)=v_j$, $i \neq j$. The edge from $v_i$ to $v_j$ is in the flow $\Phi$, so $d(v_i,v_j)<\delta$, however it is not in the flow $\Phi'$ so $d'(v_i,v_j)\ge \delta$. However, as $d'(v_i,v_j) \le d(v_i,v_j) < \delta$, we have that $v_i$ has at least one admissible edge. By the definition of $\preceq_E$, $v_i$ cannot be critical for $\Phi$, that is, a unique maximal edge that is admissible must exist. 

Finally, since $v_i \sim_\mathcal{P} v_j$ and $d'$ is a $\mathcal{P}$-transformation, we have
$
	d'(v_i,v_j) \le d(v_i,v_j),
$
and we arrive to a contradiction. 
\end{proof}

\end{document}